\documentclass{article}
\usepackage[left=1in, top=1in, bottom=1in, right=1in]{geometry}
\usepackage{natbib}
\usepackage[utf8]{inputenc} % allow utf-8 input
\usepackage[T1]{fontenc}    % use 8-bit T1 fonts
\usepackage{hyperref}       % hyperlinks
\usepackage{url}            % simple URL typesetting
\usepackage{booktabs}       % professional-quality tables
\usepackage{amsfonts}       % blackboard math symbols
\usepackage{nicefrac}       % compact symbols for 1/2, etc.
\usepackage{microtype}      % microtypography
\usepackage{xcolor}         % colors
\usepackage{enumitem}
\usepackage{multirow}
\usepackage{authblk}
\usepackage[pdftex]{graphicx}

\title{A Statistical Framework for Data-dependent Retrieval-Augmented Models}
\date{}

\author{Soumya Basu $^\ast$ \quad Ankit Singh Rawat $^\ast$ \quad Manzil Zaheer \thanks{Equal contribution in alphabetical order.}
}
\affil{ Google, New York \\
{\small \texttt{\{basusoumya,ankitsrawat,manzilzaheer\}@google.com}}
 \vspace{-3mm}
}

\usepackage{amsmath,amsfonts,bm,bbm,dsfont,amssymb,amsthm}
\usepackage[mathscr]{eucal}
\usepackage{soul}

\usepackage[textsize=tiny]{todonotes}

\usepackage[normalem]{ulem}
\usepackage{soul}
\usepackage{comment}
\newcommand{\Hquad}{\hspace{0.5em}} 

\def\1{\bm{1}}

\newcommand{\DSf}{\mathsf{D}}

\newcommand{\DCal}{\mathscr{D}}
\newcommand{\ECal}{\mathscr{E}}
\newcommand{\FCal}{\mathscr{F}}

\newcommand{\ICal}{\mathscr{I}}
\newcommand{\LCal}{\mathscr{L}}

\newcommand{\SCal}{\mathscr{S}}

\newcommand{\XCal}{\mathscr{X}}
\newcommand{\YCal}{\mathscr{Y}}
\newcommand{\ZCal}{\mathscr{Z}}

\newcommand{\RF}{\mathfrak{R}}

\newcommand{\EC}{\mathcal{E}}

\newcommand{\NC}{\mathcal{N}}

\DeclareMathAlphabet{\mathsfit}{\encodingdefault}{\sfdefault}{m}{sl}
\SetMathAlphabet{\mathsfit}{bold}{\encodingdefault}{\sfdefault}{bx}{n}

\def\sN{{\mathbb{N}}}

\newcommand{\R}{\mathbb{R}}

\DeclareMathOperator*{\argmax}{arg\,max}
\DeclareMathOperator*{\argmin}{arg\,min}

\theoremstyle{plain}
\newtheorem{theorem}{Theorem}[section]
\newtheorem{proposition}[theorem]{Proposition}

\theoremstyle{definition}
\newtheorem{definition}[theorem]{Definition}
\newtheorem{assumption}[theorem]{Assumption}
\theoremstyle{remark}

\begin{document}
\maketitle

%%%%%%%%%%%%%%%%%%%%%%%%%%%%%%%%%%%%%%%%%%%%%%%%%%%%%%%%%%%%%%%%%%%%%%%%%%%%%%%%%%%%
%%%%%%%%%%%%%%%%%%%%%%%%%%%%%%%%%%%%%%%%%%%%%%%%%%%%%%%%%%%%%%%%%%%%%%%%%%%%%%%%%%%%
% Abstract
%%%%%%%%%%%%%%%%%%%%%%%%%%%%%%%%%%%%%%%%%%%%%%%%%%%%%%%%%%%%%%%%%%%%%%%%%%%%%%%%%%%%
%%%%%%%%%%%%%%%%%%%%%%%%%%%%%%%%%%%%%%%%%%%%%%%%%%%%%%%%%%%%%%%%%%%%%%%%%%%%%%%%%%%%
\begin{abstract}
Modern ML systems increasingly augment input instances with additional relevant information to enhance final prediction. Despite growing interest in such retrieval-augmented models, their fundamental properties and training are not well understood. We propose a statistical framework to study such models with two components: 1) a {\em retriever} to identify the relevant information out of a large corpus via a data-dependent metric; and 2) a {\em predictor} that consumes the input instances along with the retrieved information to make the final predictions. We present a principled method for end-to-end training of both components and draw connections with various training approaches in the literature. Furthermore, we establish excess risk bounds for retrieval-augmented models while delineating the contributions of both retriever and predictor towards the model performance.We validate the utility of our proposed training methods along with the key takeaways from our statistical analysis on open domain question answering task where retrieval augmentation is important.
\end{abstract}

%%%%%%%%%%%%%%%%%%%%%%%%%%%%%%%%%%%%%%%%%%%%%%%%%%%%%%%%%%%%%%%%%%%%%%%%%%%%%%%%%%%%
%%%%%%%%%%%%%%%%%%%%%%%%%%%%%%%%%%%%%%%%%%%%%%%%%%%%%%%%%%%%%%%%%%%%%%%%%%%%%%%%%%%%
% Introduction
%%%%%%%%%%%%%%%%%%%%%%%%%%%%%%%%%%%%%%%%%%%%%%%%%%%%%%%%%%%%%%%%%%%%%%%%%%%%%%%%%%%%
%%%%%%%%%%%%%%%%%%%%%%%%%%%%%%%%%%%%%%%%%%%%%%%%%%%%%%%%%%%%%%%%%%%%%%%%%%%%%%%%%%%%
% \vspace{-3mm}
\section{Introduction}
\label{sec:intro}

Recent advancements in machine learning (ML) have not only led to breakthroughs on long-standing challenging tasks across various fields,  but they have also inspired a great deal of interest  to develop ML models that can solve even harder tasks~\citep{meinhardt2022trackformer,lewkowycz2022solving,cramer2021alphafold2} or focus on completely new fields~\citep{austin2021program,openai2023gpt4tr,singhal2023large}. While scaling the size of \textit{parametric} ML models, such as neural networks, is becoming the predominant approach to meet such demands~\citep{brown2020language,chowdhery2022palm,touvron2023llama,dosovitskiy2021an,dehghani2023scaling}, the excellent performance realized by this approach is marred by 
drawbacks such as high computational cost, inefficient storage of world knowledge in parameters, lack of transparency in model behavior, and reduced grounding/factuality 
of model predictions.  

Recognizing these shortcomings, \textit{retrieval-augmented models} (RAMs) have emerged as a promising alternative. Such models typically employ two components, namely \textit{retriever} and \textit{predictor}, during inference on a given input instance: The retriever first identifies instance-specific relevant information from a data-store, and then the predictor jointly processes the retrieved information and the input instance to make a final prediction. 
In practice, RAMs have enjoyed favorable performance vs. compute trade-off~\citep{retro,das2021cbr,thai2023machine} as employing moderate-size parametric models as retriever and predictor in a RAM often matches or exceeds the performance of a much larger standalone ML model that directly maps input instances to predictions. Similarly, conditioning prediction on the retrieved information has shown to exhibit improved grounding~\citep{shuster2021retrieval,lin2023ra,asai2023self}. 
Furthermore, having access to an external corpus can obviate the need to store task-specific world knowledge in model parameters and enable incorporating dynamically evolving knowledge~\citep{izacard2022atlas,pmlr-v162-liska22a}.

Despite these desirable characteristics, training RAMs presents multiple challenges.
The natural approach of independently training retriever and predictor can be sub-optimal~\citep{izacard2022atlas}. Moreover, it requires collecting intermediate supervision on the instance-dependent relevant information to retrieve, which is missing in common datasets and expensive to obtain in general.
A common strategy to circumvent the lack of intermediate supervision is to perform end-to-end training which presents its own unique challenges in the context of RAMs. 
Fundamentally, the retrieval corresponds to the non-differentiable discrete operation of selecting relevant information from a data-store, e.g., via top-k selection based on retriever scores, which prevents direct gradient propagation to the entire retriever. Several clever solutions to above-mentioned issues have been proposed in the literature that focus on different training objectives to propagate learning signal from the predictor into the retriever.
However, a formal study that unifies these solutions is missing from the literature. 

Another key challenge that prevents the resource-efficient development and deployment of RAMs is the limited understanding of their basic properties such as their generalization behavior and expressive power. For instance, how do the retriever and predictor components interact to ensure good task-specific performance? Are there any principles guiding the selection of the retriever and predictor components? 
How does (size of) the data-store feature in the final performance of a RAM?

In this paper, we address both aforementioned shortcoming in the literature pertaining RAMs. To unify the training of RAMs, we begin with writing down the natural objective function, which somehow has eluded the literature. This natural objective simply minimizes the expected prediction loss, where the expectation is taken over the distribution induced by the retriever. Empirically, we find this objective to be effective on standard benchmarks: NaturalQuestions~\citep[NQ;][]{kwiatkowski2019natural} and TriviaQA~\citep{joshi2017triviaqa}.

As for the generalization and expressive power, we present an excess risk bound for RAMs that captures the effect of retrieval and prediction function classes. The proposed bound allows us to highlight how retriever and predictor components play complementary roles to reduce approximation error as we increase their respective function class complexity.
We also capture the role of data store in improving the model performance by reducing the approximation error.
On the generalization front, we carefully decouple the generalization term in the excess risk over the predictor and retriever function classes.
This allows us to tightly control the generalization term with only {logarithmic} dependence on the data store size.
As a concrete instantiation for our excess risk bounds, we consider feed-forward neural networks of varying depth for both the retriever and the predictor.

To summarize, our main contributions include:
\vspace{-2mm}
\begin{itemize}[leftmargin=4mm, itemsep=1mm, partopsep=0pt,parsep=0pt]
    \item We present a principled objective for end-to-end training of RAMs focusing on a classification setting (Sec.~\ref{sec:problem}~\ref{sec:excess-risk}) and draw connections between existing approaches for training RAMs (Sec.~\ref{sec:connections}). 
    \item We derive excess risk bound highlighting the role played by retriever and predictor functions classes as well as {the data-store towards ensuing improved performance by RAMs  (Sec.~\ref{sec:final-risk-bound})};
    capturing the trade off between model capacities at retriever and predictor (Sec.~\ref{sec:mlp-tradeoff}).
    \item We validated the utility of the proposed objective on two standard QA benchmarks:~NaturalQuestions~(NQ) and TriviaQA (Sec.~\ref{sec:experiment}). 
\end{itemize}
% \input{010_intro}

%%%%%%%%%%%%%%%%%%%%%%%%%%%%%%%%%%%%%%%%%%%%%%%%%%%%%%%%%%%%%%%%%%%%%%%%%%%%%%%%%%%%
%%%%%%%%%%%%%%%%%%%%%%%%%%%%%%%%%%%%%%%%%%%%%%%%%%%%%%%%%%%%%%%%%%%%%%%%%%%%%%%%%%%%
% Related work (can go to end)
%%%%%%%%%%%%%%%%%%%%%%%%%%%%%%%%%%%%%%%%%%%%%%%%%%%%%%%%%%%%%%%%%%%%%%%%%%%%%%%%%%%%
%%%%%%%%%%%%%%%%%%%%%%%%%%%%%%%%%%%%%%%%%%%%%%%%%%%%%%%%%%%%%%%%%%%%%%%%%%%%%%%%%%%%
% \input{020_related_work}

%%%%%%%%%%%%%%%%%%%%%%%%%%%%%%%%%%%%%%%%%%%%%%%%%%%%%%%%%%%%%%%%%%%%%%%%%%%%%%%%%%%%
%%%%%%%%%%%%%%%%%%%%%%%%%%%%%%%%%%%%%%%%%%%%%%%%%%%%%%%%%%%%%%%%%%%%%%%%%%%%%%%%%%%%
% Problem setup and background
%%%%%%%%%%%%%%%%%%%%%%%%%%%%%%%%%%%%%%%%%%%%%%%%%%%%%%%%%%%%%%%%%%%%%%%%%%%%%%%%%%%%
%%%%%%%%%%%%%%%%%%%%%%%%%%%%%%%%%%%%%%%%%%%%%%%%%%%%%%%%%%%%%%%%%%%%%%%%%%%%%%%%%%%%
\section{Problem setup}
\label{sec:problem}

In this paper, we focus on developing a systematic understanding of RAMs with learned retrievers in a classification setting where the model has access to a data-store. Towards this, we begin by formally defining the problem setup and providing the necessary background along with the notations used.

Let's first consider the standard classification setting which requires predicting a class in $\YCal$ for a given instance $x \in \XCal$.
Assume that $\DSf_{XY}$ captures the underlying data distribution and one has access to $n$ training examples $\SCal_n \triangleq \{(x_i, y_i)\}_{i \in [n]}$ that are independent and identically distributed (i.i.d.) according to $\DSf_{XY}$. Given $\SCal_n$, one hopes to learn a classifier $f: \XCal \to \R^{|\YCal|}$ that minimizes the miss-classification error:
\begin{align}
    R(f) = \mathbb{P}_{(X, Y) \sim \DSf_{XY}}\big[\argmax_{y \in \YCal}f^y(X) \neq Y\big],
\end{align}
where $f^y(x)$ denotes the score that $f$ assigns to the $y$-th class, given the input instance $x$.
Since directly optimizing the miss-classification error or $0/1$-loss poses computational challenges, one typically selects the classifier that minimizes the empirical risk associated with a well behaved surrogate loss function $\ell: \R^{|\YCal|} \times \YCal \to \mathbb{R}$ on the training sample $\SCal_n$:
\begin{align}
    R_{\ell, n}(f) = \frac{1}{n}\sum_{i \in [n]}\ell\big(f(x_i), y_i\big).
\end{align}
The (population) risk associated with the surrogate loss function takes the following form:
\begin{align}
    R_{\ell}(f) = \mathbb{E}_{(X, Y) \sim \DSf_{XY}}\big[\ell\big(f(X), Y\big)\big].
\end{align}
Different from the standard classification setup described above, we now consider the classification task with access to a data-store: Given an instance $x$, the classifier can potentially leverage a data-store $\ICal \subseteq \ZCal$ -- a collection of potentially relevant information or evidences, where $\ZCal$ denotes the space of all possible evidences. Accordingly, one can define the empirical and population risks of a classifier $f(\cdot, \ICal): \XCal \to \R^{|\YCal|}$ as follows:
\begin{align}
    R_{\ell, \ICal, n}(f) &= \frac{1}{n}\sum_{i \in [n]}\ell\big(f(x_i, \ICal), y_i\big), \\
    R_{\ell, \ICal}(f) &= \mathbb{E}\big[\ell\big(f(X, \ICal), Y\big)\big], \label{eq:prrisk_rbm}
\end{align}
where expectation is take over in $(X, Y) \sim \DSf_{XY}$ as well as the possible randomness in $f(\cdot, \ICal)$. However, due its prohibitive computational cost, such a general classifier that directly processes the entire data-store for each prediction is far from how an additional data-store is utilized by ML models in practice.

This motivates us to study the following explicit retrieval-augmented classification setup to utilize the data-store: Given an input instance $x \in \XCal$, one first retrieves input-dependent supporting evidences $\EC^x \subset \ICal$ with the help of a {\em retriever model} which has access to the entire data-store $\ICal$.
Now, given $x$ and $\ECal^x$, one invokes a {\em predictor model} to predict the class associated with $x$. Thus, a retriever-augmented classification setup consists of two key components models: 1) retriever model and 2) predictor model, which we formally introduce next.

\noindent \textbf{Retriever model.}~For the retrieval stage, we rely on a {\em retriever model} to capture the relevance of an evidence $z \in \ICal$ towards the input instance $x \in \XCal$. Let $r_{\theta}: \XCal \times \ZCal \to \R$ be the retriever model parameterized by $\theta \in \Theta$ that assigns a relevance score $r_{\theta}(x, z)$ to the instance-evidence pair $(x, z)$. Furthermore, for each instance $x$, the retriever model $r_{\theta}$ induces the following distribution over the set of potential evidences:
% \vspace{-2mm}
\begin{align}
\label{eq:rdistn}
    p_{\theta, \ICal}\big(z | x\big) = \frac{\exp\big(r_{\theta}(x, z)\big)}{\sum_{z' \in \ICal}\exp\big(r_{\theta}(x, z')\big)}, \quad \forall~z\in\ICal.
    % \vspace*{-2mm}
\end{align}
There are multiple strategies to construct the set of input-dependent supporting evidences $\ECal^x$ based on $r_{\theta}$. For example, for a fixed integer $k \geq 1$, one could select $k$ evidences corresponding to the $k$ highest scores in $\{r_{\theta}(x, z)\}_{z\in \ICal}$. Another strategy is to sample $k$ evidences according to the distribution $ p_{\theta, \ICal}(\cdot|x)$ in \eqref{eq:rdistn}. Here, one could perform the sampling with or without replacement. In what follows, we denote the retrieved supporting evidence for the instance $x$ as $\ECal^{x}_{\theta}$ to highlight the dependence on the underlying retriever model.   

\noindent \textbf{Predictor model.}~Let $h_{\xi}: \XCal \times \ICal^{\ast} \to \R^{|\YCal|}$ be the predictor model parameterized by $\xi \in \Xi$, where $\ICal^{\ast}$ denotes the Kleene star on $\ICal$. Given $x \in \XCal$ and $\EC \in \ICal^{\ast}$, the predictor model $h_{\xi}$ assigns a score to each class in $\YCal$, defining a distribution over $\YCal$ as follows:
% \vspace{-2mm}
\begin{align}
\label{eq:hdistn}
    p_{\xi}\big(y | x, \ECal) = \frac{\exp\big(h^{y}_{\xi}(x, \EC)\big)}{\sum_{y' \in \YCal}\exp\big(h^{y'}_{\xi}(x, \EC)\big)}, \quad \forall~y \in \YCal,
\end{align}
where $h^{y}_{\xi}(\cdot, \cdot)$ denotes the score assigned to the $y$-th class by the predictor model $h_{\xi}$.

For ease of exposition, we focus on the setting with $k = |\ECal^x_{\theta}|=1, \forall x \in \XCal,$ in our analysis throughout this paper. This corresponds to retrieving a single supporting evidence for each input instance. Our analysis can be generalized to $k > 1$ by working with a $\tilde{\ICal} \subseteq \ICal^{k}$ as the new data-store and $\tilde{p}_{\theta, \ICal}(\cdot |x)$ as a distribution over $\tilde{\ICal}$ obtained by suitably modifying $p_{\theta, \ICal}$ in \eqref{eq:rdistn}. 
For example, when $k$ supporting evidences are sampled with replacement, then the following holds $\forall (z_1,\ldots, z_k) \in \ICal^k$.
\begin{align*}
    \tilde{p}_{\theta, \ICal}\big((z_1,\ldots, z_k) \big\vert x\big) = \prod_{j\in [k]} p_{\theta, \ICal}(z_j | x).
\end{align*}
% }
% \vspace{-5mm}

\textbf{Empirical risk minimization and excess risk for RAMs.}~
For a pair of retriever and predictor models parameterized by $\theta$ and $\xi$, respectively, we can define the empirical and population risks associated with a (surrogate) loss function $\ell$ as follows:
% \vspace{-2mm}
\begin{align}
\label{eq:erisk_ram}
     R_{\ell, \ICal, n}(\xi, \theta) &= \frac{1}{n}\sum_{i \in [n]}\sum_{z \in \ICal}p_{\theta}(z|x)\ell\big(h_{\xi}(x_i, z), y_i\big), \\
    R_{\ell, \ICal}(\xi, \theta) &= \mathbb{E}\big[\ell\big(h_{\xi}(X, \EC^{X}_{\theta}), Y\big)\big]. \label{eq:prisk_ram}
    % \vspace{-2mm}
\end{align}
Note that the expectation in \eqref{eq:prisk_ram} is taken over $(X, Y) \sim \DSf_{XY}$ as well as the randomness involved in the retrieval stage, e.g., sampling the evidences according to $p_{\theta, \ICal}(\cdot | x)$ in \eqref{eq:rdistn}. Given a pair of predictor class $\Xi$ and retriever class $\Theta$, let $(\hat{\xi}, \hat{\theta})$ denote the predictor-retriever pair obtained via \textit{empirical risk minimization} (ERM) as follows:
% \vspace{-1mm}
\begin{align}\label{eq:erm}
  (\hat{\xi}, \hat{\theta}) \in \argmin_{(\xi, \theta) \in \Xi \times \Theta} R_{\ell, \ICal, n}(\xi, \theta).
\end{align} 
% \vspace{-5mm}

Let $\FCal_{\rm all}$ denote the set of all measurable functions from $\XCal \times \ZCal$ to $\R^{|\YCal|}$. The optimal risk for the classification with access to the data-store is achieved by the best possible predictor $f^{\ell}_{\rm opt, \ICal} \in \FCal_{\rm all}$ when it has access to the best retrieved evidence in $\ICal$. In particular, we have % \vspace{-1mm}
\begin{align}
\label{eq:fellopt}
f^{\ell}_{\mathrm{opt},\ICal} = \argmin_{f\in \FCal_{\mathrm{all}}} \mathbb{E}\big[\min_{z\in \ICal}\ell(f(X,z), Y)\big].%\nonumber
\end{align}
% \vspace{-5mm}

Given $f^{\ell}_{\mathrm{opt},\ICal}$, we defined the \textit{excess} risk of a predictor-retriever pair $(\xi, \theta)$ as follows:
% \vspace{-1mm}
\begin{align}\label{eq:excess-ell}
 &\Delta_{\ell, \ICal}(\xi, \theta) = R_{\ell, \ICal}(\xi, \theta) - R_{\ell, \ICal}(f^{\ell}_{\mathrm{opt}, \ICal}) 
 \triangleq R_{\ell, \ICal}(\xi, \theta) - \mathbb{E}\big[\min_{z\in \ICal}\ell(f^{\ell}_{\mathrm{opt},\ICal}(X,z), Y)\big]. 
\end{align}
% \vspace{-5mm}

With the formal definition of the classification setting with access to a data-store and the necessary background in place, 
we proceed to address the two key objectives of this work:
1) Proposing a natural and efficient joint end-to-end training procedure for the predictor-retriever pair in a RAM; and 2) Developing a rigorous statistical understanding of RAMs focusing on the interaction between predictor and retriever components towards reducing overall excess risk.

%%%%%%%%%%%%%%%%%%%%%%%%%%%%%%%%%%%%%%%%%%%%%%%%%%%%%%%%%%%%%%%%%%%%%%%%%%%%%%%%%%%%
%%%%%%%%%%%%%%%%%%%%%%%%%%%%%%%%%%%%%%%%%%%%%%%%%%%%%%%%%%%%%%%%%%%%%%%%%%%%%%%%%%%%
% Algorithm
%%%%%%%%%%%%%%%%%%%%%%%%%%%%%%%%%%%%%%%%%%%%%%%%%%%%%%%%%%%%%%%%%%%%%%%%%%%%%%%%%%%%
%%%%%%%%%%%%%%%%%%%%%%%%%%%%%%%%%%%%%%%%%%%%%%%%%%%%%%%%%%%%%%%%%%%%%%%%%%%%%%%%%%%%
% \input{040_algorithm_new}
% \input{040_algo_new}

%%%%%%%%%%%%%%%%%%%%%%%%%%%%%%%%%%%%%%%%%%%%%%%%%%%%%%%%%%%%%%%%%%%%%%%%%%%%%%%%%%%%
%%%%%%%%%%%%%%%%%%%%%%%%%%%%%%%%%%%%%%%%%%%%%%%%%%%%%%%%%%%%%%%%%%%%%%%%%%%%%%%%%%%%
% Analysis
%%%%%%%%%%%%%%%%%%%%%%%%%%%%%%%%%%%%%%%%%%%%%%%%%%%%%%%%%%%%%%%%%%%%%%%%%%%%%%%%%%%%
%%%%%%%%%%%%%%%%%%%%%%%%%%%%%%%%%%%%%%%%%%%%%%%%%%%%%%%%%%%%%%%%%%%%%%%%%%%%%%%%%%%%

\section{Joint training and excess risk}
\label{sec:excess-risk}

Recall that training a RAM involves training both the retriever $r_{\theta}: \XCal \times \ZCal \to \R$ and the predictor $h_{\xi}: \XCal \times \ICal \to \R^{|\YCal|}$ components of the model 
without access to intermediate supervision on retrieval, which is infeasible to obtain in most practical settings. 
Thus, it becomes critical to devise methods to jointly train $r_{\theta}$ and $h_{\xi}$ with access to only labeled instances $\SCal_n = \{(x_i, y_i)\}_{i \in [n]} \subseteq \XCal \times \YCal$ with the predictor guiding the retriever training based on how valuable the retriever-provided evidences are towards the correct final prediction.

Towards this, we leverage the empirical risk from \eqref{eq:erisk_ram} along with the 
log-loss $\ell(h_{\xi}(x, z), y) = -\log p_{\xi}(y | x, z)$,
where $p_{\xi}(y|x, z)$ is defined in \eqref{eq:hdistn}. In particular, this leads to the following joint end-to-end training objective:
\begin{align}
\label{eq:the_obj}
\LCal_n(\xi, \theta; \ICal) \triangleq R_{{\rm log}, \ICal, n}(\xi, \theta)  = - \frac{1}{n} \sum_{i \in [n]}\sum_{z \in \ICal} p_{\theta, \ICal}(z | x_i)\cdot \log p_{\xi}(y_i | x_i, z).
\end{align}
% \vspace{-5mm}

Note that the objective in \eqref{eq:the_obj} aims to improve the end-to-end performance of a RAM in deployment in the sense that the objective aims to minimize the expected loss given the selected evidences as per the retriever-induced distribution.
One can use gradient-based methods to jointly minimize the objective in \eqref{eq:the_obj} with respect to $(\xi, \theta)$; however, its efficient implementation is non-trivial due to the sum over entire data-store $\ICal$. In App.~\ref{sec:effcient}, we discuss some approximate design choices.
Lastly, please refer to Sec.~\ref{sec:connections} for connections between our proposed objective in \eqref{eq:the_obj} and some of the existing end-to-end training approaches for RAMs.

Next, to study the generalization and expressive power of RAMs, we want to bound the excess risk $\Delta_{\ell, \ICal}(\hat{\xi}, \hat{\theta})$ as defined in~\eqref{eq:excess-ell}. 
We consider $\XCal$ to be a compact subspace of $\mathbb{R}^{d_x}$ and, for simplicity, take $\XCal \subseteq [-1, 1]^{d_x}$. Similarly, we consider that each retrieval example $z \in \ICal$ is embedded in the space $[-1, 1]^{d_z}$. We consider a data-store that polynomially scales with training data size, i.e., $|\ICal| = {\rm poly}(n)$. 
For the purpose of analysis, we specialize our log-loss to be bounded by $\ell_{\max} > 0$, which is given as
% \vspace{-2mm}
\begin{align}\label{eq:bounded-cross-entropy-loss}
&\ell(h_{\xi}(x,z), y) = \min(\ell_{\max}, - \log p_{\xi}(y|x,z)) =  \min\bigg(\ell_{\max}, \log\Big(\sum_{y'\in \YCal}\exp(h^{y'}_{\xi}(x,z))\Big) - h^{y}_{\xi}(x,z)\bigg), \nonumber 
\end{align}
% \vspace{-2mm}
where $p_{\xi}(y|x,z)$ and $h^{y}_{\xi}(x,z)$ are defined in~\eqref{eq:hdistn}.

\subsection{Excess risk decomposition}
Our excess risk relies on separating out the contribution coming from the retriever and the predictor during the joint training. Moreover, the retriever and predictor errors can be each split into generalization and approximation error.  

The population risk optimizer of our joint training over the space $\Xi \times \Theta$ is defined as
% \vspace{-2mm}
\begin{align*}
&\xi^{\ast}_{\rm joint}, \theta^{\ast}_{\rm joint}  = \argmin_{(\xi, \theta) \in \Xi \times \Theta} \mathbb{E}_{X}\big[\mathbb{E}_{Z\sim p_{\theta}(\cdot| X)} \mathbb{E}_{Y|X} \ell\big(h_{\xi}(X, Z), Y)\big].  
\end{align*}

For a predictor $\xi$, sample $x \in \XCal$ and retrieved example $z \in \ICal$, let us denote the risk averaged over the labels $\YCal$ as 
% \vspace{-1mm}
\begin{equation}\label{eq:label-avg-risk}
   g_{\xi}(x,z) = \mathbb{E}_{Y|X=x}[\ell\big(h_{\xi}(x, z), Y)]. 
\end{equation}
For any fixed predictor $\xi$ (\textit{not necessarily} in $\Xi$) and fixed data-store $\ICal$, the retriever that optimizes the joint population risk is given as $p^{\ast, \xi}(z|x) = \mathbbm{1}_{\argmin_{z'\in \ICal} g_{\xi}(x, z')}(z)$, where a tie is broken arbitrarily. Note that, for each sample $x$, the best retrieved evidence $z$ may change. We define the optimal predictor \textit{within the class} $\Xi$ with best possible retriever as
% \vspace{-2mm}
$$
\xi^{\ast} = \argmin_{\xi \in \Xi} \mathbb{E}_{X}\big[\min_{z \in \ICal} g_{\xi}(X, z)\big].
$$
The optimal retriever \textit{within the class} $\Theta$ for a given predictor $\xi$ is defined as 
%\vspace{-2mm}
$$
\theta(\xi) = \argmin_{\theta \in \Theta} \mathbb{E}_{X}\big[\mathbb{E}_{Z\sim p_{\theta}(\cdot| X)} g_{\xi}(X, Z)\big].
$$

The excess risk for the classes $\Theta$ and $\Xi$ can be bounded as 
\begin{align}\label{eq:joint-excess-bound}
& \Delta_{\ell, \ICal}(\hat{\xi}, \hat{\theta}) \leq \underbrace{\sum_{(\theta, \xi) \in \{(\hat{\theta}, \hat{\xi}), (\theta^{\ast}_{\rm joint}, \xi^{\ast}_{\rm joint})\}}|R_{\ell, \ICal}(\xi, \theta) - R_{\ell, \ICal, n}(\xi, \theta)|}_{\text{Generalization Error}} \nonumber \\
& \qquad + \underbrace{R_{\ell, \ICal}(\xi^*, \theta(\xi^{\ast})) - \mathbb{E}_{X}\big[\min_{z \in \ICal} g_{\xi^*}(X, z)\big]}_{\text{retriever error}} + \underbrace{\mathbb{E}_{X}\big[\min_{z \in \ICal} g_{\xi^*}(X, z)\big]  - R_{\ell, \ICal}(f_{{\rm opt}, \ICal}^{\ell})}_{\text{predictor error}}
\end{align}

\subsection{Generalization error}
We first bound the generalization error and relate it to the covering number of the retriever and predictor class.

As our loss is bounded by $\ell_{\max}$, through standard concentration bounds~\citep{shalev2014understanding}, we obtain that, for any $\delta > 0$, with probability at least $(1- \delta)$:
$$|R_{\ell, \ICal}(\xi^{\ast}_{\rm joint}, \theta^{\ast}_{\rm joint}) - R_{\ell, \ICal, n}(\xi^{\ast}_{\rm joint}, \theta^{\ast}_{\rm joint})| \leq 3 \ell_{\max}\sqrt{\tfrac{\log(1/\delta)}{n}}.$$

However, $(\hat{\xi}, \hat{\theta})$ is learned from the data. A high probability generalization error requires taking union over the space of $\Xi \times \Theta$. We employ Rademacher complexity based generalization error bounds. Next, the covering number of the space $\Xi$ is used to bound the associated Rademacher complexity. See \citet{shalev2014understanding} for details.  

We define two norms which are used in defining the covering numbers for $\Theta$ and $\Xi$. In particular,~$\forall \mathbf{u} \in \mathbb{R}^{n \times |\ICal|}$ {and fixed} $\xi \in \Xi, \theta \in \Theta$,
\begin{align}\label{eq:cov-norm}
 &\|\mathbf{u}\|_{2, [n], \xi} = \Big(\tfrac{1}{n} \sum_{i\in [n]} \big(\sum_{z \in \ICal}u_{i,z}\ell\big(h_{\xi}(x_i, z), y_i\big)\big)^2 \Big)^{1/2},\nonumber\\
 &\|\mathbf{u}\|_{2, [n], \theta} = \Big(\tfrac{1}{n} \sum_{i\in [n]} \big(\sum_{z \in \ICal}p_{\theta}(z|x_i)u_{i,z}\big)^2 \Big)^{1/2}.
\end{align}
We also define $\mathcal{N}(\Xi, \nu, {\|\cdot\|_{2, [n], \theta}})$ to be the $\nu$-covering number  for the class $\Xi$ with respect to the norm $\|\cdot\|_{2, [n], \theta}$, and $\mathcal{N}(\Theta, \nu, {\|\cdot\|_{2, [n], \xi}})$ to be the $\nu$-covering number  for the class $\Theta$ with respect to the norm $\|\cdot\|_{2, [n], \xi}$. Then we have the generalization bound given as 
\begin{align}
     &|R_{\ell, \ICal}(\hat{\xi}, \hat{\theta}) - R_{\ell, \ICal, n}(\hat{\xi}, \hat{\theta})| \leq \inf_{\varepsilon \in [0, \ell_{\max}/2]} \Big( 8\varepsilon 
     + \tfrac{24}{\sqrt{n}} \int_{\varepsilon}^{\tfrac{\ell_{\max}}{2}} f_{\mathcal{N}}(\nu/2; \Theta, \Xi) + f_{\mathcal{N}}(\nu/2; \Xi, \Theta) d \nu\Big),
\end{align}
% \vspace{-5mm}
%
for $f_{\mathcal{N}}(\nu; \mathcal{A}, \mathcal{B}) = \sup_{b \in \mathcal{B}}\sqrt{\log(\mathcal{N}(\mathcal{A}, \nu, \|\cdot\|_{2, [n], b}))}.$

We use ideas in \citet{zhang2023mathematical} to upper bound the covering number with pseudo-dimension (defined in the Appendix~\ref{sec:prelims}) of the function class. This allows us to have a $\log |\ICal|$ dependence in the generalization error, while working with norm unbounded function classes.

\subsection{Approximation error}
We next proceed to bound the retriever and predictor approximation errors. Towards this, we extensively use the Sobolev functions spaces. A Sobolev space for a domain $\Omega$ is characterized by two quantities, $\kappa$ -- the number of weak-derivatives a (real-valued) function within it possesses, and $L_p(\Omega)$ -- the norm with respect to which these derivatives are integrable. 
Please see Appendix~\ref{sec:prelims} for a complete definition.

\subsubsection{Retriever error}
The retriever error is given by how well the score function $r_{\theta}(x,z)$ approximates the optimal retriever given $\xi^*$. In order to do so we first need to impose some smoothness constraints on the function $g_{\xi^*}: \XCal \times \ZCal \to \mathbb{R}$. In particular, we assume the following.
\begin{assumption}[Complexity of $g_{\xi^*}$]\label{a:gap-sobolev-main}
There exists a baseline function $b_{\xi^*}: [-1, 1]^{d_x} \to \mathbb{R}$ such that the function $\mathrm{gap}_{\xi^*}:[-1, 1]^{d_x + d_z}  \to \mathbb{R}$ defined by $\mathrm{gap}_{\xi^*}(x, z)\triangleq(g_{\xi^*}(x,z) - b_{\xi^*}(x))$ lies in the Sobolev space with $\kappa$ derivatives and $L_{\infty}([-1, 1]^{d_x + d_z})$ norm.
\end{assumption}
The above assumption says that for the predictor $\xi^*$ the loss profile (averaged over labels in $\YCal$) $g_{\xi^*}(x, z)$, has two components --  a (possibly) complex $b_{\xi^*}(x)$ component that is uniform over $z$, and a `smooth' $\mathrm{gap}_{\xi^*}(x,z)$ component. In other words, given two similar retrieved evidences, the predictor incurs similar losses when each of the evidences is utilized with an input instance.

Then, for any $\tau > 0$, we can bound the retriever loss as follows: 
\begin{align}\label{eq:reteiver-loss}
    &R_{\ell, \ICal}(\xi^*, \theta(\xi^{\ast})) - \mathbb{E}_{X}\big[\min_{z \in \ICal} g_{\xi^*}(X, z)\big] 
    \leq \inf_{\theta \in \Theta} \ell_{\max} \|r_{\theta} + \tau\cdot \mathrm{gap}_{\xi^*}\|_{\infty} + \frac{\log |\ICal|}{\tau^2}
\end{align}

\subsubsection{Predictor error}
The predictor error is measured with the optimal retrieval (as the retriever error is considered separately above). For this, we need to first quantify how the retrieval augmentation using the data-store $\ICal$ helps. 

\paragraph{Usefulness of retrieval set:}
We start with characterization of the prediction task in the presence of the data-store $\ICal \subset \ZCal$.  We assume that there exists a score function $h_{*}: \XCal \times \ZCal \to \mathbb{R}^{|\YCal|}$, and the corresponding probability distribution 
\begin{equation}\label{eq:score-function-main}
    p_{*}^y(x, z) = \frac{\exp(h_{*}^y(x,z))}{\sum_{y'} \exp(h_{*}^{y'}(x,z))},
\end{equation}
that approximates  $p_{\DSf_{XY}}^y(x) := \mathbb{P}_{Y \sim \DSf_{Y|X}}(y|X=x)$ well for all $x \in \XCal$ and $y \in \YCal$. Furthermore, we want this score function $h_{*}$ to lie coordinate wise in a Sobolev space. The following assumption formalizes this. % captures the above intuition. 
\begin{assumption}[Retrieval quality]\label{a:true-sobolev-main}
There exists a score function $h_{*}: \XCal \times \ZCal \to \mathbb{R}^{|\YCal|}$ such that
% \vspace{-3mm}
\begin{enumerate}[leftmargin=4mm, itemsep=1mm, partopsep=0pt,parsep=0pt]
    \item for each $y \in \YCal$, the function $h_{*}^y$ lies in the Sobolev space with $\kappa_{\ICal}$ derivatives and finite $L_{\infty}([-1, 1]^{d_x + d_z})$ norm,
    \item for any $x \in \XCal$, there exists a retrieved evidence $z^*(x) \in \ICal$ such that $p_{*}^y(x, z)$, as defined in~\eqref{eq:score-function-main}, satisfies
    % \vspace{-1mm}
    $$\max_{y \in \YCal}\sup_{x\in \XCal}|p_{*}^y(x, z^*(x)) - p_{\DSf_{XY}}^y(x)| \leq c_{\ICal} |\ICal|^{-\gamma_{\ICal}}.$$ 
%
% \vspace{-3mm}
\end{enumerate}
% \vspace{-4mm}
\end{assumption}
Note that this is independent of the retriever class $\Theta$ and $\Xi$, and captures intrinsic property of the data-store $\ICal$. The tuple $(\gamma_{\ICal}, d_z, \kappa_{\ICal})$ defines the usefulness of $\ICal$. In particular, the higher $\gamma_{\ICal}$ the closer the approximation; and the higher the $\kappa_{\ICal}$ and smaller the embedding dimension $d_z$ the `simpler' the score function used for this
approximation.

Under the Assumption~\ref{a:true-sobolev-main}, we bound the predictor error as 
\begin{align}\label{eq:predictor-loss}
    \mathbb{E}_{X}\big[\min_{z \in \ICal} g_{\xi^*}(X, z)\big]  - R_{\ell, \ICal}(f_{{\rm opt}, \ICal}^{\ell}) &\leq  \inf_{\xi \in \Xi} 2  \mathbb{E}_{X}\big[ \max_{y \in \YCal} |h_{\xi}^y(X, z^*(X)) - h_{*}^y(X, z^*(X))|\big] + \nonumber \\
    &\qquad \quad (|\YCal| - 1)\exp(-\ell_{\max}) + c_{\ICal} |\ICal|^{- \gamma_{\ICal}}\exp(\ell_{\max}). 
\end{align}
{
One key step in arriving to the above inequality is expressing the loss of $f_{{\rm opt}, \ICal}^{\ell}$ using the probability function $h_{*}$ defined in Assumption~\ref{a:true-sobolev-main}.} In particular, under Assumption~\ref{a:true-sobolev-main}, we show that
\begin{align*}
\mathbb{E}_{X}\big[\min_{z\in \ICal} g_{f_{{\rm opt}, \ICal}^{\ell}}(X, z)\big] \geq \mathbb{E}_{X}\big[g_{h_{*}}(X, z^*(X))\big] - 
(|\YCal| - 1)\exp(-\ell_{\max}) - c_{\ICal} |\ICal|^{- \gamma_{\ICal}}\exp(\ell_{\max}). 
\end{align*}

\subsection{Final excess risk bound}
\label{sec:final-risk-bound}
We now combine the three components of the excess risk bounds under Assumptions~\ref{a:gap-sobolev-main} and \ref{a:true-sobolev-main} and discuss the design tradeoffs. % Our main result is as follows:
The following theorem captures our main theoretical result.
\begin{theorem}[Excess risk of joint training]\label{thm:main}
Under Assumption~\ref{a:gap-sobolev-main} and \ref{a:true-sobolev-main}, the excess risk for the retriever class $\Theta$ and predictor class $\Xi$ is bounded as 
\begin{align*}
    % \vspace{-3mm}
    &\Delta_{\ell, \ICal}(\hat{\xi}, \hat{\theta})\leq 3 \ell_{\max} (\tfrac{1}{n} + \sqrt{\tfrac{\log(n)}{n}}) + \inf_{\varepsilon \in [0, \tfrac{\ell_{\max}}{2}]} 8\varepsilon + \tfrac{24}{\sqrt{n}} \int_{\varepsilon}^{\tfrac{\ell_{\max}}{2}} f_{\mathcal{N}}(\tfrac{\nu}{2}; \Theta, \Xi) + f_{\mathcal{N}}(\tfrac{\nu}{2}; \Xi, \Theta) d \nu \nonumber \\
    & \qquad+ \inf_{\theta \in \Theta}\inf_{\tau > 0} \ell_{\max} \|r_{\theta} + \tau \cdot \mathrm{gap}_{\xi^*}\|_{\infty} + \frac{\log |\ICal|}{\tau^2} \nonumber \\
    & \qquad + \inf_{\xi \in \Xi} 2  \mathbb{E}_{X}\big[ \max_{y \in \YCal} |h_{\xi}^y(X, z^*(X)) - h_{*}^y(X, z^*(X))|\big]   + (|\YCal| - 1)\exp(-\ell_{\max})  + c_{\ICal} |\ICal|^{- \gamma_{\ICal}}\exp(\ell_{\max}),
\end{align*}
where $f_{\mathcal{N}}(\nu; \mathcal{A}, \mathcal{B})\triangleq\sup_{b \in \mathcal{B}}\sqrt{\log(\mathcal{N}(\mathcal{A}, \nu, \|\cdot\|_{2, [n], b}))}$ and $\|\cdot\|_{2, [n], \theta}$ and $\|\cdot\|_{2, [n], \xi}$ are defined in~\eqref{eq:cov-norm}.
\end{theorem}

\subsection{Illustrative example: MLPs}
\label{sec:mlp-tradeoff}
We instantiate both our retriever and predictor classes to be %deep neural network, aka 
multi-layer perceptron (MLP) with depth $L_{\rm ret}$ \& width $W_{\rm ret} = O(d_x + d_z)$ and depth $L_{\rm pred}$ \& width $W_{\rm pred} = O(|\YCal|(d_x + d_z))$, respectively.  The class ${\rm MLP}\left(\mathbb{R}^{d}, \mathbb{R}^k;  L, W\right)$ is defined in Appendix \ref{sec:prelims}. The specialized excess risk bound for this setting is given as
\begin{theorem}[Excess risk for  MLP]\label{thm:main-mlp}
Under Assumption~\ref{a:gap-sobolev-main} and \ref{a:true-sobolev-main}, the excess risk for the retriever class $\Theta = MLP\left(\mathbb{R}^{d_x + d_z}, \mathbb{R};  L_{\rm ret}, O(d_x + d_z)\right)$ and predictor class $\Xi= MLP\left(\mathbb{R}^{d_x + d_z}, \mathbb{R}^{|\YCal|};  L_{\rm pred}, O(|\YCal|(d_x + d_z))\right)$ is bounded as 
% \vspace{-2mm}
\begin{align*}
    % \vspace{-3mm}
    &\Delta_{\ell, \ICal}(\hat{\xi}, \hat{\theta}) \leq  \tilde{O}\left(\frac{\ell_{\max}}{\sqrt{n}} \left(L_{\rm ret} +  L_{\rm pred}|\YCal| \right) \right) +  O\Big(\ell_{\max}L_{\rm ret}^{- \tfrac{4\kappa}{3(d_x + d_z)}}\log^{1/3}(|\ICal|)\Big) + \nonumber \\
    &   O\left(L_{\rm pred}^{-\tfrac{2\kappa_{\ICal}}{(d_x + d_z)}} +  (|\YCal| - 1)\exp(-\ell_{\max}) + c_{\ICal} |\ICal|^{- \gamma_{\ICal}} \exp(\ell_{\max})\right).
\end{align*}
\end{theorem}
% \vspace{-4mm}

Finally, to capture the optimal trade-off under finite data size $n$, we consider classes of retriever and predictors that change with the data size, denoted by $\Theta_n$ and $\Xi_n$, with growing depths $L_{{\rm ret}, n}$ and $L_{{\rm pred},n}$ respectively. Similarly, we also consider growing upper bound on the loss function by $\ell_{\max, n}$. %, as well as growing $\ell_{\max, n}$. 
Let $d_{\rm tot} = d_x + d_z$. For $L_{{\rm ret}, n} = n^{\tfrac{3d_{\rm tot}}{6d_{\rm tot} + 8 \kappa}}$, $L_{{\rm pred}, n} = (\sqrt{n}/|\YCal|)^{\tfrac{d_{\rm tot}}{2 d_{\rm tot} + 4 \kappa_{\ICal}}}$, and $\ell_{\max, n} = \log |\YCal|  + \frac{\kappa_{\ICal}}{(d_{\rm tot} + 2 \kappa_{\ICal})}\log n $, the excess risk is bounded by 
$$
O\bigg(n^{- \tfrac{2\kappa}{3d_{\rm tot} + 4 \kappa}} + \max\Big(|\ICal|^{-\gamma_{\ICal}} |\YCal| n^{\tfrac{\kappa_{\ICal}}{d_{\rm tot} + 2 \kappa_{\ICal}}}, 
\big(\frac{n}{|\YCal|^2}\big)^{- \tfrac{\kappa_{\ICal}}{d_{\rm tot} + 2 \kappa_{\ICal}}}\Big)\bigg).$$
% \vspace{-5mm}

We should contrast the above result with the prediction when there is no retrieval. Let us assume that the functions $p_{\DSf_{XY}}^y(x)$ for all $y \in \YCal$ lies in the Sobolev space with derivative $\kappa_{\rm true}$ and $L_{\infty}$ norm. The predictor excess risk rate with $L_{{\rm pred}, n} = (\sqrt{n}/|\YCal|)^{\frac{d_x}{d_x + 2 \kappa_{\rm true}}}$ is $O((n/ |\YCal|^2)^{- \tfrac{\kappa_{\rm true}}{d_x + 2 \kappa_{\rm true}}})$.

Note that our analysis indicates that we may  {\em gain through retrieval}: For large enough data store $|\ICal| \geq  |\YCal|^{\tfrac{d_{\rm tot} \gamma_{\ICal}^{-1}}{d_{\rm tot} + 2 \kappa_{\ICal}}} n^{\tfrac{2\kappa_{\ICal} \gamma_{\ICal}^{-1}}{d_{\rm tot} + 2 \kappa_{\ICal}}}$, as the data size $n$ increases and  
$\kappa >  \tfrac{3 d_{\rm tot}}{2 d_x} \kappa_{\rm true}$ and $\kappa_{\ICal} > \tfrac{d_{\rm tot}}{d_x} \kappa_{\rm true}$ (see Fig.~\ref{fig:my_label}).

\begin{figure}
    \vspace{-2mm}
    \centering
    \hfill
    \includegraphics[width=0.4\linewidth,trim=6mm 7mm 3mm 4mm,clip]{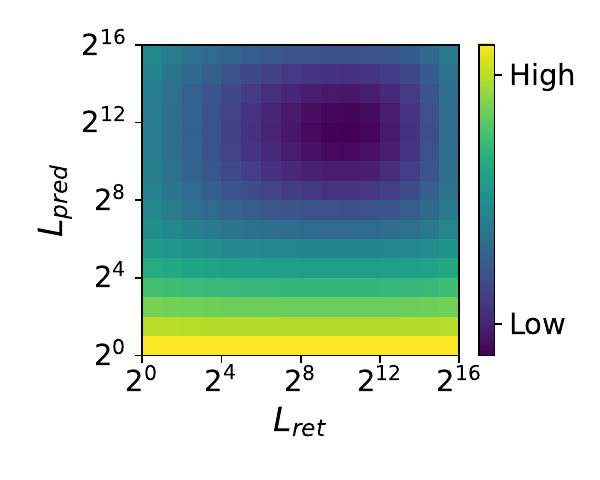}\hfill
    \includegraphics[width=0.4\linewidth]{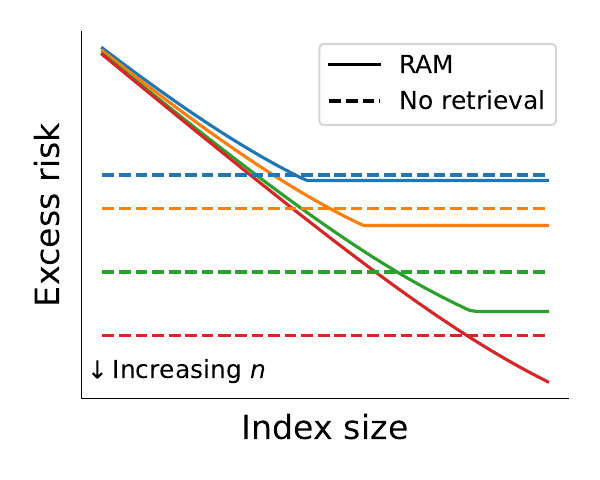}
    \hfill
    \vspace{-3mm}
    \caption{\textbf{Left}: Excess risk bound as we vary retriever and predictor size for a fixed $n$ and $\ICal$ based on Theorem~\ref{thm:main-mlp}. Note that different size combination of predictor and retriever  achieves same risk bound.
    \textbf{Right}: Excess risk bound of RAM as we increase data-store size in contrast to direct MLP predictor with no retrieval. We plot for various values of $n$, with each color corresponding to a fixed $n$.}
    \label{fig:my_label}
    \vspace{-3mm}
\end{figure}

\begin{table*}[t]
    % \small
    \centering
    \scalebox{0.98}{
    \renewcommand{\arraystretch}{0.95}
    \begin{tabular}{@{}lccc ccc ccc ccc@{}}
    \toprule
     \multicolumn{1}{@{}l}{\multirow{2.5}{*}{Method}} &  & \multicolumn{3}{c}{small} &                      & \multicolumn{3}{c}{base} &                      & \multicolumn{3}{c}{large} \\ \cmidrule(lr){3-5} \cmidrule(lr){7-9} \cmidrule(l){11-13} 
    &  & small   & base   & large  &  & small   & base  & large  &  & small   & base   & large  \\ \midrule
% \multicolumn{9}{@{}l}{Fixed retriever $\theta_0$, train reader $\xi$} \\[1mm]
% \qquad Cross-Entropy
%  &  & 0.23  & 0.27  & 0.28  &  & 0.28  & 0.32  & 0.35  &  & 0.32  & 0.36  & 0.38 \\
% \midrule
% \multicolumn{9}{@{}l}{Fixed predictor $\xi^\star(\theta_0)$, train retriever $\theta$} \\[1mm]
% \qquad EMDR2 &  & 0.24  & 0.29  & 0.31  &  & 0.29  & 0.34  & 0.37  &  & 0.33  & 0.38  & 0.41 \\
% \qquad PDist &  & 0.30  & 0.35  & 0.38  &  & 0.34  & 0.40  & 0.43  &  & 0.38  & 0.43  & 0.44 \\
% \qquad Cross-Entropy + PG &  & 0.26  & 0.31  & 0.32  &  & 0.31  & 0.36  & 0.38  &  & 0.36  & 0.40  & 0.41 \\
% \qquad Cross-Entropy + TopK &  & 0.29  & 0.35  & 0.38  &  & 0.34  & 0.40  & 0.43  &  & 0.37  & 0.42  & 0.45 \\
% \midrule
% \multicolumn{9}{@{}l}{Jointly train predictor $\xi$ and retriever $\theta$} \\[1mm]
% \qquad EMDR2 &  & 0.24  & 0.30  & 0.33  &  & 0.30  & 0.36  & 0.39  &  & 0.35  & 0.40  & 0.42 \\
% \qquad PDist &  & 0.29  & 0.33  & 0.37  &  & 0.33  & 0.37  & 0.39  &  & 0.36  & 0.40  & 0.42 \\
% \qquad Cross-Entropy + PG &  & 0.27  & 0.31  & 0.33  &  & 0.33  & 0.37  & 0.38  &  & 0.36  & 0.40  & 0.41 \\
% \qquad Cross-Entropy + TopK &  & 0.33  & 0.38  & 0.40  &  & 0.37  & 0.42  & 0.45  &  & 0.39  & 0.44  & 0.46 \\
%
\multicolumn{9}{@{}l}{No retriever, train predictor $\xi$} \\[0.5mm]
\Hquad Reverse Cross-Entropy & & & 19.6 & & & & 25.5 & & & & 29.1\\
\midrule
\multicolumn{9}{@{}l}{Fixed retriever $\theta_0$, train predictor $\xi$} \\[0.5mm]
\Hquad Reverse Cross-Entropy
 &  & 23.2  & 26.6  & 28.3  &  & 27.5  & 32.4  & 34.7  &  & 32.2  & 36.4  & 37.8 \\
\midrule
\multicolumn{9}{@{}l}{Fixed predictor $\xi^\star(\theta_0)$, train retriever $\theta$} \\[0.5mm]
\Hquad EMDR2 &  & 23.9  & 28.5  & 31.0  &  & 29.2  & 34.2  & 36.6  &  & 33.4  & 38.0  & 40.8 \\
\Hquad PDist &  & 30.1  & 34.5  & 38.4  &  & 34.0  & 39.7  & 42.8  &  & 37.6  & 42.8  & 44.7 \\
\Hquad Reverse Cross-Entropy + PG &  & 25.9  & 30.6  & 31.7  &  & 31.5  & 36.4  & 37.9  &  & 36.0  & 40.2  & 41.4 \\
\Hquad Reverse Cross-Entropy + TopK &  & 29.4  & 35.5  & 37.9  &  & 33.8  & 39.7  & 43.0  &  & 37.2  & 42.3  & 45.0 \\
\midrule
\multicolumn{9}{@{}l}{Jointly train predictor $\xi$ and retriever $\theta$} \\[0.5mm]
\Hquad EMDR2 &  & 24.1  & 30.4  & 32.7  &  & 30.4  & 35.6  & 39.3  &  & 34.5  & 39.7  & 42.1 \\
\Hquad PDist &  & 28.7  & 33.2  & 36.6  &  & 33.3  & 37.1  & 38.8  &  & 36.2  & 40.2  & 41.6 \\
\Hquad Reverse Cross-Entropy + PG &  & 27.1  & 31.0  & 32.7  &  & 33.3  & 37.2  & 38.2  &  & 36.5  & 39.8  & 41.4 \\
\Hquad Reverse Cross-Entropy + TopK &  & 32.8  & 37.8  & 40.1  &  & 36.6  & 41.8  & 44.8  &  & 38.8  & 43.8  & 46.4 \\
    \bottomrule
    \end{tabular}}
    \vspace{-1mm}
    \caption{\textbf{Exact match accuracy on NQ}. We measure the performance of RAMs across various training paradigms and model sizes. Top row specifies the predictor size and the second row specifies the retriever size.}
    \label{tab:nq_accuracy}
    %\vspace{-5mm}
\end{table*}

\subsection{Connections with prior end-to-end training}
\label{sec:connections}

We conclude our treatment of end-to-end training of RAMs by drawing parallels between our proposed method with some representative approaches from the literature.
% \textbf{End-to-end training of Multi-Document Reader and Retriever 

\textbf{EMDR$^2$}~\citet{sachan2021emdr2} 
minimize the following objective based on the negative log-likelihood:
\begin{align}
\label{eq:emdr2}
&\LCal^{\textsc{Emdr}^2}_{n}(\xi, \theta; \ICal) = -\frac{1}{n}\sum_{i \in [n]} \log p_{\xi, \theta, \ICal}(y|x)  =  -\frac{1}{n}\sum_{i \in [n]} \log \Big(\sum_{z \in \ICal} p_{\theta, \ICal}(z | x_i) \cdot p_{\xi}(y_i | x_i, z)\Big).
\end{align}
It follows from the convexity of $-\log(\cdot)$ and Jensen's inequality that our objective in \eqref{eq:the_obj} upper bounds the EMDR$^2$ objective in \eqref{eq:emdr2}; as a result, minimizing the former also minimizes the latter but not vice versa.

\textbf{Perplexity distillation (PDist)}~Another approach for joint training of RAMs in the literature involves optimizing two distinct objectives for training the predictor and retriever components. For example, \citet{izacard2022atlas} propose multiple objectives for retriever training, including PDist~\citep{sachan2023questions} which is defined as follows:
\begin{align}
\label{eq:pdist}
&\LCal^{\textsc{PDist}}_{\ICal, n}(\theta; \xi, \ICal) = \frac{1}{n}\sum_{i\in[n]}\mathrm{CE}\big(p^{\textsc{PDist}}_{\xi, \ICal}(Z|x_i, y_i), p_{\theta, \ICal}(Z|x_i)\big),
\end{align}
where $\mathrm{CE}(\cdot, \cdot)$ denotes the 
cross entropy between two distributions
and 
\begin{align*}
p^{\textsc{PDist}}_{\xi, \ICal}(z|x, y) = { p_{\xi}(y|x, z)}/{\sum_{z' \in \ICal} p_{\xi}(y|x, z')}\quad \forall~z\in \ICal,
\end{align*}
represents a predictor-assigned distribution over evidences based on their utility towards making correct prediction. As for the predictor training, they optimize an objective akin to \eqref{eq:the_obj} with respect to $\xi$. Besides this similarity in the predictor training, our approach for retrieval training has a subtle connection with PDist.
Note that PDist optimizes \textit{forward cross-entropy} between the predictor and the retriever induced distributions to train the retriever. On the other hand, our objective in \eqref{eq:the_obj} is closer to $\frac{1}{n}\sum_{i}\mathrm{CE}(p_{\theta, \ICal}(Z|x_i), p^{\textsc{PDist}}_{\xi, \ICal})$, the \text{reversed cross-entropy} between the two distributions.  
The former has the ``mean-seeking'' behavior whereas the latter has the ``mode-seeking'' behavior~\citep{huszar2015not,gu2023knowledge,agarwal2023gkd}.

\textbf{Similarity with RLHF/RLAIF}~
Note that the per-example objective of our retrieval training approach takes the form: 
\begin{align}
\mathbb{E}_{Z \sim p_{\theta, \ICal}(\cdot | x_i)}\big[\ell\big(h_{\xi}(x_i, Z), y_i\big)\big],    
\end{align}
i.e., the predictor model provides feedback on the (value) of the evidences sampled by the retriever model. Alternatively, one can view $-\ell\big(h_{\xi}(x_i, Z), y_i\big)$ as the reward assigned to the evidence $z$ by the predictor model $h_{\xi}$ and retriever model aims to select those evidences that maximize this reward value. This is similar to RLHF~\citep{ziegler2019fine} or RLAIF~\citep{bai2022constitutional} paradigm, where the underlying LLM aims to sample those generations which maximize the reward assigned by a reward model. However, note that in RLHF/RLAIF paradigm the policy network and reward model are not jointly trained together unlike in RAM.

%%%%%%%%%%%%%%%%%%%%%%%%%%%%%%%%%%%%%%%%%%%%%%%%%%%%%%%%%%%%%%%%%%%%%%%%%%%%%%%%%%%%
%%%%%%%%%%%%%%%%%%%%%%%%%%%%%%%%%%%%%%%%%%%%%%%%%%%%%%%%%%%%%%%%%%%%%%%%%%%%%%%%%%%%
% Experiments
%%%%%%%%%%%%%%%%%%%%%%%%%%%%%%%%%%%%%%%%%%%%%%%%%%%%%%%%%%%%%%%%%%%%%%%%%%%%%%%%%%%%
%%%%%%%%%%%%%%%%%%%%%%%%%%%%%%%%%%%%%%%%%%%%%%%%%%%%%%%%%%%%%%%%%%%%%%%%%%%%%%%%%%%%
\section{Experiments}
\label{sec:experiment}

There have been numerous successful practical applications of RAMs in the literature (e.g.,~\citet{sachan2021emdr2,izacard2022atlas}). 
Here, we present a brief empirical study for such models in order to corroborate the benefits predicted by our theoretical results.
In particular,
we consider the task of open-domain question answering and show that proposed objective is competitive to the objectives proposed in the literature and observe the trade-offs in model capacity between retriever and predictor model.

\textbf{Data}~
Our evaluation is based on two benchmark datasets: NQOpen~\cite{kwiatkowski2019natural} and TriviaQA~\cite{joshi2017triviaqa}, which serve as sources for supervised examples $(x,y)$, while chunked Wikipedia 2018 is used as the data-store  $\ICal$ following literature~\citep{Karpukhin:2020}.
Consistent with established practices, we employ the exact match metric to assess the correspondence between the predicted answers and the ground truth. Additionally, we introduce a recall metric to measure the frequency at which the answer string appears within the retrieved documents.

\textbf{Models}~
We implement the retriever component using GTR~\citep{ni2022large} and the predictor component using T5~\citep{raffel2020exploring}.
We sweep across small, base, and large configurations for both retriever and predictor. The details regarding the model sizes, expressed in terms of the number of parameters, are provided in Table~\ref{tab:params} (App.~\ref{app:expt}).

\begin{table*}[!htbp]
    \centering
    \scalebox{0.95}{
    \renewcommand{\arraystretch}{0.98}
    \begin{tabular}{@{}lccc ccc ccc ccc@{}}
    \toprule
     \multicolumn{1}{@{}l}{\multirow{2.5}{*}{Method}} &  & \multicolumn{3}{c}{small} &                      & \multicolumn{3}{c}{base} &                      & \multicolumn{3}{c}{large} \\ \cmidrule(lr){3-5} \cmidrule(lr){7-9} \cmidrule(l){11-13} 
    &  & small   & base   & large  &  & small   & base  & large  &  & small   & base   & large  \\ \midrule
\multicolumn{9}{@{}l}{No retriever, train predictor $\xi$} \\[1mm]
\Hquad Reverse Cross-Entropy & & & 17.9 & & & & 23.1 & & & & 28.0\\
\midrule
\multicolumn{9}{@{}l}{Fixed retriever $\theta_0$, train predictor $\xi$} \\[1mm]
\Hquad Reverse Cross-Entropy
 &  & 31.5  & 34.9  & 38.8  &  & 37.0  & 40.6  & 44.4  &  & 43.4  & 45.9  & 49.7 \\
\midrule
\multicolumn{9}{@{}l}{Fixed predictor $\xi^\star(\theta_0)$, train retriever $\theta$} \\[1mm]
\Hquad EMDR2 &  & 34.6  & 41.3  & 48.3  &  & 40.1  & 48.2  & 53.4  &  & 46.0  & 50.7  & 54.9 \\
\Hquad PDist &  & 45.7  & 53.3  & 57.2  &  & 50.8  & 53.2  & 61.6  &  & 53.5  & 55.4  & 62.3 \\
\Hquad Reverse Cross-Entropy + PG &  & 43.2  & 46.7  & 54.3  &  & 48.6  & 56.1  & 55.1  &  & 51.7  & 56.4  & 56.7 \\
\Hquad Reverse Cross-Entropy + TopK &  & 43.6  & 50.4  & 54.4  &  & 48.6  & 54.9  & 58.5  &  & 52.1  & 56.6  & 60.3 \\
\midrule
\multicolumn{9}{@{}l}{Jointly train predictor $\xi$ and retriever $\theta$} \\[1mm]
\Hquad EMDR2 &  & 37.0  & 43.1  & 49.7  &  & 42.4  & 50.5  & 55.6  &  & 47.1  & 53.4  & 59.2 \\
\Hquad PDist &  & 46.7  & 54.3  & 57.3  &  & 48.8  & 56.7  & 60.7  &  & 51.0  & 58.5  & 63.3 \\
\Hquad Reverse Cross-Entropy + PG &  & 47.0  & 52.9  & 55.7  &  & 49.9  & 57.6  & 61.1  &  & 52.1  & 59.8  & 59.2 \\
\Hquad Reverse Cross-Entropy + TopK &  & 46.8  & 52.9  & 56.0  &  & 49.2  & 56.6  & 60.1  &  & 52.3  & 58.8  & 62.4 \\
    \bottomrule
    \end{tabular}}
    \caption{\textbf{Exact match accuracy on TriviaQA}. We measure the performance of RAMs across various training paradigms and model sizes. Top row specifies the predictor size and the second row specifies the retriever size.}
    \label{tab:tqa_accuracy}
\end{table*}

\textbf{Methods}~
We compare following approaches: 1) utilizing no retriever, directly training predictor, 2) employing a fixed retriever, but training the predictor, 3) using a fixed predictor, but training the retriever, and 4) conducting joint training of both components.
For the joint training and the retriever training phases, we experiment with multiple objectives: EMDR2 (cf.~\eqref{eq:emdr2}), PDist (cf.~\eqref{eq:pdist}), Reverse Cross-Entropy + PG (cf.~\eqref{eq:ce-pg} in App.~\ref{sec:effcient}), and Reverse Cross-Entropy + TopK (cf.~\eqref{eq:ce-topk} in App.~\ref{sec:effcient}).
Efficiently implementing any of these objectives is challenging due to the need to compute the gradient with respect to expectation over the entire data-store.
We consider two approaches for computing the gradients approximately by: 1) restricting the expectation to top-K elements similar to EMDR2 and PDist; and 2) using REINFORCE~\citep{williams1992simple} to obtain an unbaised estimate. More details can be found in App.~\ref{sec:effcient}.

\textbf{Observation 1}~
The addition of a retrieval component markedly enhances performance, as demonstrated in Tables~\ref{tab:nq_accuracy} and \ref{tab:tqa_accuracy}, which present the exact match accuracy. Further improvements are observed when the retriever is specifically trained while keeping the predictor fixed. Joint training emerges as the most effective strategy.

\textbf{Observation 2}~
Tables~\ref{tab:nq_recall} and \ref{tab:tqa_recall} (App.~\ref{app:expt}) list the recall for the presence of the answer string within the retrieved content. PDist consistently achieves the highest recall, aligning with expectations given its design for distilling the retriever based on the predictor's scores. However, despite its superior recall, other objectives may lead to better overall performance than PDist, suggesting that different objectives optimize the retriever and predictor with varying efficiencies.

\textbf{Observation 3}~
Finally, in Table~\ref{tab:timing}, we report the query per second (QPS), as a proxy for computational cost, achieved by different configuration of retriever and predictor model sizes. 
For achieving a specific accuracy threshold (e.g., $\geq$38.8 on NQ), multiple configurations are viable, such as pairing a large predictor with a small retriever, a base model for both, or a small predictor with a large retriever. The associated query per second (QPS) rates for these configurations are 135, 333, and 800, respectively, illustrating that equivalent accuracy levels can be attained with significantly differing QPS rate.
This corroborates with our trade-offs in excess risk bounds for MLPs with different capacity in retriever and predictor components as illustrated in Figure~\ref{fig:my_label}.
Thus, adding capacity to different parts of the model has different repercussion on quality and computational cost.

%%%%%%%%%%%%%%%%%%%%%%%%%%%%%%%%%%%%%%%%%%%%%%%%%%%%%%%%%%%%%%%%%%%%%%%%%%%%%%%%%%%%
%%%%%%%%%%%%%%%%%%%%%%%%%%%%%%%%%%%%%%%%%%%%%%%%%%%%%%%%%%%%%%%%%%%%%%%%%%%%%%%%%%%%
% Discussion
%%%%%%%%%%%%%%%%%%%%%%%%%%%%%%%%%%%%%%%%%%%%%%%%%%%%%%%%%%%%%%%%%%%%%%%%%%%%%%%%%%%%
%%%%%%%%%%%%%%%%%%%%%%%%%%%%%%%%%%%%%%%%%%%%%%%%%%%%%%%%%%%%%%%%%%%%%%%%%%%%%%%%%%%%
% \vspace{-1mm}
\section{Discussion and related work}
\label{sec:rw}
% \vspace{-1mm}

\begin{table*}[!th]
    \centering
    \begin{tabular}{@{}cccc cccc cccc@{}}
    \toprule
      \multicolumn{3}{c}{small} & &
      \multicolumn{3}{c}{base}  & &
      \multicolumn{3}{c}{large} \\ 
      \cmidrule{1-3} \cmidrule(lr){5-7} \cmidrule(l){9-11} 
    small   & base   & large  &  & small   & base  & large  &  & small   & base   & large  \\ \midrule
    822.60  & 819.83  & 800.89  &  & 334.30  & 333.22  & 331.06  &  & 135.06  & 135.34  & 134.87 \\
    \bottomrule
    \end{tabular}
    \caption{\textbf{Query per second}. We measure the query per second processed by RAMs as a proxy for computational cost across various  model sizes. Top row specifies the predictor size and the second row specifies the retriever size.}
    \label{tab:timing}
\end{table*}

Several works have proposed some form of retrieval augmented models.
Here, we provide a brief account of the evolution of RAMs and discuss how our proposed joint-learning objective and the framework for excess risk analysis compare with existing end-to-end training methods.

\textbf{Augment with local neighborhood}~
The first approaches dating back to 1970s employed just augmenting training instance in the local neighborhood of the input space~\citep{stone1977consistent,stone1980optimal}.
Such approaches gained a lot of attention as parametric regression was not adequate in various practical applications of the time.
This line of work aims to fit a low-degree polynomial at each point in the data set based on a subset of data points,
which resulted in a rich literature on local polynomial regression in low dimensions
~\citep{katkovnik1979dynamic,cleveland1979robust,pinsker1980optimal,donoho1988automatic,ruppert1994local,ibragimov2013statistical}.
These classical ideas have found their application in many ML algorithms such as face recognition~\citep{jain2011online}, dimensionality reduction via local linear embeddings~\citep{roweis2000nonlinear}, domain adaptation~\citep{yang2021exploiting}, test time training on neighboring points~\citep{sun2020test,gandelsman2022test}, etc.
Recently,~\citet{basu23a} generalized this setup of augmenting with a local neighborhood of the input instance in the context of modern ML models like neural networks %and kernel methods 
and proposed a statistical framework to study such retrieval-based models. 
However, they \textit{do not} consider a learned or a specialized distance metric to find the augmenting set, which is critical for realizing good performance in practice~\citep{schonberger2017comparative,karpukhin2020dense} and studied in the present work. 

\textbf{Fixed retriever augmentation}~
Next generation retrieval augmented models started to deploy either a hand crafted or a learned retriever.
\citet{zhang2006svm} employed SIFT~\citep{lowe1999object} based retrieval followed by a SVM~\citep{cortes1995support} classifier to improve performance on multiple vision tasks.
\citet{chen2009similarity} studied generalization bounds for SVM-kNN methods -- one of the limited works in this domain with formal analysis.
For natural language understanding, 
methods like TF-IDF~\citep{sparck1972statistical} were employed in the tasks like case based reasoning~\citep{Leake1996AcquiringCA} and open-domain question answering (ODQA;~\citealt{voorhees1999trec}). Unlike many previous methods, one retrieves relevant text passages in ODQA settings as opposed to retrieving labelled training pairs.
With introduction of transformers~\citep{vaswani2017attention}, both retriever and predictor models based on encoder and decoder, respectively, have become popular across various domains, including image classification~\citep{long2022retrieval,iscen2023improving}, text classification~\citep{wang_more_valuable,zemlyanskiy-etal-2022-generate}, ODQA~\citep{lee-etal-2019-latent,izacard-grave-2021-leveraging}, language modelling~\citep{retro}, and even protein folding prediction~\citep{cramer2021alphafold2}. 
Even using the same transformer model
as both retriever and predictor boosts performance in language modeling~\citep{khandelwal2020knnlm}. 
Unlike SVM-kNN~\citep{chen2009similarity}, to best of our knowledge, a formal analysis of retrieval-augmented approaches with modern neural networks is missing from the literature. 
Interestingly, retrieving examples also helps in-context learning~\citep{rubin-etal-2022-learning, pmlr-v202-li23l}.
Our framework covers this scenario with $z$ representing the in-context examples retrieved from a data-store of examples.
Our risk bounds can provide insights into why in-context learning with \textit{retrieved} few-shot examples performs better than a zero-shot model.

\textbf{End-to-end trained retriever augmentation}~
For ODQA, \citet{guu2020realm} proposed maximizing the marginalized likelihood by considering the retrieved set as a latent variable. EMDR2~\citep{sachan2021emdr2} optimized the same objective by approximating it based on the retriever induced distribution on the elements that receive top-K scores by the retriever.
Hindsight~\citep{paranjape2022hindsight} instead optimizes the ELBO by introducing a variational distribution with access to the outputs.
VOD~\citep{lievin23a} further generalized the standard ELBO based on KL divergence by employing R\'{e}nyi divergence thereby tightening the lower bound. 
On the other hand, Atlas~\citep{izacard2022atlas} proposed an auxiliary loss for training the retriever directly 
rather than following the latent variable approach.
Interestingly, RAG~\citep{lewis2020retrieval} proposed to only train the query encoder for retriever, leaving the retrieval index fixed, thereby alleviating much of the end-to-end training difficulties of RAMs, but at cost of limiting model adaptation flexibility.
None of these prior works studied statistical properties vis-\`a-vis expressivity and generalization of RAMs. 

\section{Conclusion}
In this work, we initiate the development of a theoretical framework to study the statistical properties of RAMs with data-dependent retrieval. 
Our excess-risks analysis allows us to highlight how retriever and predictor components play complementary roles in reducing approximation error
as we increase their respective function class complexity. 
We surface both theoretically and empirically a Pareto surface achieving the same performance with different size predictors and retrievers.
As future work, it would be interesting to study the effect of dynamically updatable data-store and multi-step retrievals for making predictions.

\bibliographystyle{plainnat}
\bibliography{ref}

% \onecolumn
% % {\color{red} KEEPING ONE COLUMN FOR EASE OF WRITING FOR NOW.}

%%%%%%%%%%%%%%%%%%%%%%%%%%%%%%%%%%%%%%%%%%%%%%%%%%%%%%%%%%%%%%%%%%%%%%%%%%%%%%%
%%%%%%%%%%%%%%%%%%%%%%%%%%%%%%%%%%%%%%%%%%%%%%%%%%%%%%%%%%%%%%%%%%%%%%%%%%%%%%%
% APPENDIX
%%%%%%%%%%%%%%%%%%%%%%%%%%%%%%%%%%%%%%%%%%%%%%%%%%%%%%%%%%%%%%%%%%%%%%%%%%%%%%%
%%%%%%%%%%%%%%%%%%%%%%%%%%%%%%%%%%%%%%%%%%%%%%%%%%%%%%%%%%%%%%%%%%%%%%%%%%%%%%%
\newpage
\appendix
\onecolumn

\section{Preliminaries}\label{sec:prelims}

\begin{definition}[Rademacher complexity]
Given a sample $\SCal_n = \{(x_i, y_i)\}_{i \in [n]} \subset \XCal \times \YCal$ and a real-valued function class $\FCal : \XCal \times \YCal \to \R$, the {\em empirical} Rademacher complexity of $\FCal$ with respect to $\SCal_n$ is defined as
\begin{align}
\RF_{\SCal_n}( \FCal ) = \frac{1}{n}\mathbb{E}_{\bm{\sigma}}\left[ \sup_{f \in \FCal } \sum_{i=1}^{n}\sigma_i f(x_i, y_i)\right],
\end{align}
where $\bm{\sigma} = \{\sigma_i\}_{i \in [n]}$ is a collection of $n$ i.i.d. Bernoulli random variables. For $n \in \sN$, the Rademacher complexity $\bar{\RF}_{n}( \FCal )$ and {\em worst case} Rademacher complexity $\RF_{n}( \FCal )$ are defined as follows.
\begin{align}
 \bar{\RF}_{n}( \FCal ) = \mathbb{E}_{\SCal_n \sim \DSf^n} \left[\RF_{\SCal}( \FCal ) \right], \quad \text{and} \quad \RF_{n}( \FCal ) = \sup_{\SCal_n \sim (\XCal \times \YCal)^n} \RF_{\SCal}( \FCal ).
\end{align}
\end{definition}

\begin{definition}[Covering nsumber]\label{def:covering}
Let $\epsilon > 0$ and $\|\cdot\|$ be a norm defined over $\R^n$. Given a function class $\FCal: \XCal \times \YCal \to \R$ and a collection of points $\SCal_n = \{(x_i, y_i)\}_{i \in [n]} \subset \XCal \times \YCal$, we call a set of points $\{u_j\}_{j \in [m]} \subset \R^n$ an $(\epsilon, \|\cdot\|)$-cover of $\FCal$ with respect to $\SCal$, if we have
\begin{align}
\sup_{f \in \FCal} \min_{j \in [m]} \|f(\SCal_n) - u_j\| \leq \epsilon,
\end{align}
where $f(\SCal_n) = \big(f(x_1, y_1),\ldots, f(x_n, y_n)\big) \in \R^n$. The $\|\cdot\|$-covering number $\NC(\FCal, \epsilon, \|\cdot\|; \SCal_n)$ denotes the cardinality of the minimal $(\epsilon, \|\cdot\|)$-cover of $\FCal$ with respect to $\SCal_n$. In particular, if $\|\cdot\|$ is a $\ell_p$ norm (e.g. $\|v\| = (\sum_{j=1}^{d} |v_j|^p)^{1/p}$ for $v \in \mathbb{R}^d$), then we simply use $\NC(\FCal, \epsilon,  \|\cdot\|_{L_p}; \SCal_n)$ to denote the corresponding $\ell_p$-covering number. 
\end{definition}
When $\SCal_n$ is unambiguous we may drop it, i.e., we use $\NC(\FCal, \epsilon, \|\cdot\|_{L_p})$ to represent the covering number.

\begin{definition}[Multi-layer perceptron (MLP)]\label{def:mlp}
 We consider for both retrieval and predictor, the class of multi-layer-perceptron, aka fully connected Deep Neural Network, with Relu nonlinearity $\sigma(x) = \max(x, 0)$. 
An MLP is specified by the number of layers $L$, and the width $W$.  We define with weight $\mathbf{W} \in \mathbb{R}^{d_2} \times \mathbb{R}^{d_1}$ and bias $b \in \mathbb{R}^{d_2}$, an affine transform $A_{\mathbf{W}, b}(\mathbb{R}^{d_1}, \mathbb{R}^{d_2}): x \to  \mathbf{W} x + b$. Let $\sigma \circ A_{\mathbf{W}, b}(\mathbb{R}^{d_1}, \mathbb{R}^{d_2})$ define the elementwise application of the  Relu non-linearity on the affine transform. 
The class of $L$ layers and $W$ width MLP is defined as 
\begin{equation}\label{eq:mlp-def}
   {\rm MLP}(\mathbb{R}^{d}, \mathbb{R}^{k}; W, L) = \{ A_{\mathbf{W}_L, b_L} \circ \sigma \circ A_{\mathbf{W}_{L-1}, b_{L-1}} \circ \dots \sigma \circ A_{\mathbf{W}_0, b_0}\}, 
\end{equation}
where $\mathbf{W}_{L} \in \mathbb{R}^{k \times W}$ and $b_{L} \in \mathbb{R}^{k}$; $\mathbf{W}_{i} \in \mathbb{R}^{W \times W}$ and $b_{i} \in \mathbb{R}^{W}$, for $1\leq i \leq (L-1)$; and  $\mathbf{W}_{0} \in \mathbb{R}^{W \times d}$ and $b_{0} \in \mathbb{R}^{W}$.
\end{definition}

\begin{definition}[Sobolev space]\label{def:sobolev}
For $p \geq 1$, we denote the set of functions with finite $L_p$ norm over $\Omega$ as $L_p(\Omega)$, i.e., for any $f \in L_p(\Omega)$, $\|f\|_{L_p(\Omega)} \triangleq \big(\int_{s \in \Omega} f(s)^p ds\big)^{1/p} < \infty$. Note that for $p = \infty$, we have $\|f\|_{L_\infty(\Omega)} = \mathrm{ess}\sup_{s\in \Omega} |f(s)|.$
Let $\alpha \in \mathbb{N}^{d}$ denote a multi-index, and $|\alpha| = \sum_{i\in d} \alpha_i$ be it's degree. We denote by $D^{\alpha}$ the weak-derivative with respect to multi-index $\alpha$ for any function. 

For an integer $\kappa > 0$, the Sobolev semi-norm $W^\kappa(L_{p}(\Omega))$ for a function $f$ that has weak-derivatives of order $\kappa$ is defined as 
$$
\forall 1 \leq  p < \infty, |f|_{W^\kappa(L_{p}(\Omega))} \triangleq 
\big(\sum_{\alpha: |\alpha| = \kappa} \|D^{\alpha}f\|_{L_p(\Omega)}^p \big)^{1/p} \text{ and }
|f|_{W^\kappa(L_{\infty}(\Omega))} \triangleq \max_{\alpha: |\alpha| = \kappa}\|D^{\alpha}f\|_{L_\infty(\Omega)}.
$$ 
The Sobolev norm $W^\kappa(L_{p}(\Omega))$ for the same function $f$ is defined as 
$\|f\|_{W^\kappa(L_{p}(\Omega))} = \|f\|_{L_{p}(\Omega)} +  |f|_{W^\kappa(L_{p}(\Omega))}.$
A function $f$ with all  weak-derivatives of order $\kappa$, and a finite $W^\kappa(L_{p}(\Omega))$ norm lies in the Sobolev space with $\kappa$ derivatives and $L_p(\Omega)$ norm.
\end{definition}

In our approximation guarantees for MLP retreiver and predictor classes later, we use  \cite[Theorem 1]{siegel2023optimal}. We restate the result here for completeness.

\begin{theorem}[Restated \cite{siegel2023optimal} Theorem 1] \label{thm:siegel}
Let  $f_0: \Omega \to \mathbb{R}$ be a function in the Sobolev space with $\kappa$ derivatives and norm $L_q(\Omega)$, for $q \in [1, \infty)$ and $\kappa \in (0, \infty)$. For ${\Omega = [-1,1]^d}$ and any $p \in [1, \infty)$ satisfying $(1/q - 1/p) \leq s/d$, we have for $C = c(\kappa, d) < \infty$,  and $W = 25 d + 31$
$$
\inf_{f\in {\rm MLP}(\mathbb{R}^{d}, \mathbb{R}; W, L)} \|f - f_0\|_{L_{p}(\Omega)} \leq C \|f_0\|_{W^\kappa(L_{q}(\Omega))} L^{-\tfrac{2\kappa}{d}}.
$$
\end{theorem}

Our generalization bounds leverages VC Dimension bounds of MLP \cite{bartlett2019nearly}. Here, we state some results from \cite{bartlett2019nearly} for completeness.

\begin{definition}[VC dimension and growth of a binary function class]\label{def:vc-dimension} For $\mathcal{H}$, a class of functions from $\mathcal{A}$ to $\{0,1\}$ the growth function of $\mathcal{H}$ evaluated on an input set of size $m$,  is defined as  
$$
\Pi_{\mathcal{H}}(m) = \max_{a_1, \dots, a_m \in \mathcal{A}} |\{h(a_1), \dots, h(a_m): h \in \mathcal{H}\}|. 
$$
The ${\rm VCdim}(\mathcal{H})$ is defined as the largest $m$ such that $\Pi_{\mathcal{H}}(m) = 2^m$, where if no such $m$ is there we have ${\rm VCdim}(\mathcal{H}) = \infty$.
\end{definition}

\begin{definition}[Pseudo dimension of real valued function class] \label{def:pseudo-dimension} 
Let $\mathcal{F}$ be a class of functions from some space $\mathcal{A}$ to the real $\mathbb{R}$.  The pseudo-dimension of class $\mathcal{F}$, denoted by $Pdim(\mathcal{F})$, is the largest $m$ such that there exists $\{a_1, \dots, a_m, r_1, \dots, r_m\} \in \mathcal{A}^m \times \mathbb{R}^m$ such that for any binary sequence $\{b_1, \dots, b_m\} \in \{0,1\}^m$ there exists a function $f \in \mathcal{F}$ satisfying  $\forall i: f(a_i) > r_i \iff b_i = 1$. 
\end{definition}

Note that the pseudo-dimension is same as the VC dimension of the subgraph of class $\mathcal{F}$ which is used in \cite{zhang2023mathematical}. Let $sgn(x) = \mathbbm{1}(x \geq 0)$.  We denote by $sgn(f)$ the sign of the function $f: \mathcal{A} \to \mathbb{R}$. We define $sgn(\mathcal{F}) \triangleq \{sgn(f): f \in \mathcal{F} \}$, and the VC dimension of the real valued function class $\mathcal{F}$ as ${\rm VCdim}(\mathcal{F}) \triangleq {\rm VCdim}(sgn(\mathcal{F}))$. It is mentioned in \cite{bartlett2019nearly} that for neural network with a fixed architecture and fixed activation functions, namely class ${\rm MLP}$, we have that ${\rm VCdim}(sgn({\rm MLP})) = Pdim({\rm MLP})$.

We now adapt \cite[Theorem 6]{bartlett2019nearly} to use it for the class ${\rm MLP}(\mathbb{R}^d, \mathbb{R}; L, W)$ the employs the Relu non-linearity. In terminology of \cite{bartlett2019nearly}, it amounts to focusing on the number of breakpoints $pnt=1$, and degree of polynomial $deg=1$.\footnote{Originally in \cite{bartlett2019nearly} degree is denoted by $d$ and break point by $p$, but we use $deg$ and $pnt$, respectively, to avoid confusion. These notations are used for the rest of the paper.}  

\begin{theorem}[Adaptation of \cite{bartlett2019nearly} Theorem 6]
\label{thm:bartlett}
Consider the neural network class ${\rm MLP}(\mathbb{R}^d, \mathbb{R}; L, W)$ that has the Relu non-linearity. Let $W_{total, l}$ denote the total number of parameters (weights and biases) up to layer $l \leq (L-1)$, and $k_l$ denote the number of non-linear units (output width) in layer $l$. Also define the parameters $\bar{L} = \tfrac{1}{W_{total, L}} \sum_{l=1}^{L} W_{total, l} \leq L$, and $R = \sum_{l = 1}^{L} l k_l \leq L^2 W$. Then for the function class $\mathcal{F}$ of all real-valued functions computed by the MLP class and $m$ 
$$
\Pi_{sgn(\mathcal{F})}(m) \leq \prod_{l=1}^{L} 2 \left(\frac{2 e m k_l l}{W_{total, l}}\right)^{W_{total, l}} \leq (4 e m L)^{W_{total, L}}.
$$
Moreover, we have 
$$
{\rm VCdim}(\FCal) = L + \bar{L} W_{total, L} \log_2(4e \sum_{l} l k_l \log_2(\sum_{l} 2e l k_l)) 
= O(\bar{L} W_{total, L} \log(L^2 W)).
$$
\end{theorem}

We generalize the above result to capture the MLP with multi dimensional output as used by our predictor. 
\begin{theorem}[Multi-ouput version of \cite{bartlett2019nearly} Theorem 6]
\label{thm:multioutput-bartlett}
Consider the neural network class ${\rm MLP}(\mathbb{R}^d, \mathbb{R}^k; L, W)$ that has Relu non-linearity with $W_{total, l}$, $k_l$, $\bar{L}$, and $R$ as defined in Theorem~\ref{thm:bartlett}. We denote by  $\mathcal{F}$ the class of functions  $f: \mathbb{R}^d \times [k] \to \mathbb{R}$ where $f(\cdot, k)$ is the $k$-th output coordinate of a neural network in class ${\rm MLP}(\mathbb{R}^d, \mathbb{R}^k; L, W)$. Then, we have
$$
{\rm VCdim}(\FCal) = L + \bar{L} W_{total, L} \log_2(4e \sum_{l} l k_l \log_2(\sum_{l} 2e l k_l)) 
= O(\bar{L} W_{total, L} \log(L^2 W)).
$$
\end{theorem}
\begin{proof}
Let $a \in \mathbb{R}^{W_{total, L}}$ parameterize one function $f \in \mathcal{F}$.  Based on the discussions, we need to find the ${\rm VCdim}$ of the set 
$\{sgn(f( x_i, j, a)): a \in \mathbb{R}^{W_{total, L}}, j \in [k], i \in [m] \}$.
Note that here we have $f: \mathbb{R}^{W_{total, L}} \times [m] \times [k] \to \mathbb{R}$ is a function mapping the tuple $(x_i, j, a)$ to a real number.

We obtain the following inequality.
\begin{align*}
&|\{sgn(f(x_i, j, a)): a \in \mathbb{R}^{W_{total, L}}, i \in [m], j \in [k] \}|\\
&\qquad\leq  \sum_{j \in [k]}  |\{sgn(f(x_i, j, a)): a \in \mathbb{R}^{W_{total, L}}, i \in [m]\}|\\
&\qquad \leq  \sum_{j \in [k]}  \Pi_{sgn({\rm MLP}(\mathbb{R}^{d}, \mathbb{R}; L, W))}(m) \\
&\qquad \leq k 2^L (2 e R m / {W_{total, L}})^{W_{total, L}}.
\end{align*}
In the first inequality, we partition the set with respect to $j \in [k]$. For the second inequality we notice that for a fixed $j$ the function  $f(x_i, j, a)$ is computed by ${\rm MLP}(\mathbb{R}^{d}, \mathbb{R}; L, W)$ and bound it with the growth function $\Pi_{sgn({\rm MLP}(\mathbb{R}^{d}, \mathbb{R}; L, W))}$ over $m$ points. Therefore, for the third inequality we can apply the specified bound for $\Pi_{sgn({\rm MLP}(\mathbb{R}^{d}, \mathbb{R}; L, W))}(m)$ inside the proof of Theorem 6 in \cite{bartlett2019nearly}. Note that, here we have specialized for Relu nonlinearlity, i.e. breaking point $pnt = 1$, and degree $deg = 1$. Applying Lemma 6 in \cite{bartlett2019nearly} we obtain 
$$
{\rm VCdim}(\FCal) \leq L \log(k) + W_{total, L}\log_2(4 e R\log_2(4 e R)) = O(L\log(k) + L^2W^2 \log(LW)).
$$ 
\end{proof}

Finally, we state a bounded version of the Gibb's inequality, that lower bounds the cross entropy of two discrete probability distributions.  
\begin{proposition}[Truncated Gibb's inequality]\label{prop:bounded-celoss}
Let us consider two discrete distributions $\alpha, \beta$ over alphabet size $K$. Then for any constant $C > 0$, we have 
$$
\sum_{i = 1}^{K} \alpha_i \min(C, - \log(\beta_i)) \geq \sum_{i = 1}^{K} \alpha_i \min(C, - \log(\alpha_i)) - (K - 1)\exp(-C).
$$
\end{proposition}
\begin{proof}
For two discrete distributions $\alpha, \beta$ over alphabet size $K$.
\allowdisplaybreaks
\begin{align*}
    &\sum_{i = 1}^{K} \alpha_i \min(C, - \log(\beta_i)) \\
    &= - \sum_{i = 1}^{K} \alpha_i \log(\max(\exp(-C), \beta_i))\\
    &= - \sum_{i = 1}^{K} \alpha_i \log(\alpha_i) + \sum_{i = 1}^{K} \alpha_i \log\big(\alpha_i / \max(\exp(-C), \beta_i)\big) \\
    &\geq  - \sum_{i = 1}^{K} \alpha_i \log(\alpha_i) + (\sum_{i = 1}^{K} \alpha_i) \log\big(\sum_{i = 1}^{K} \alpha_i / \sum_{i = 1}^{K} \max(\exp(-C), \beta_i))\big)\\ % tsum -> sum
    & \geq - \sum_{i = 1}^{K} \alpha_i \log(\alpha_i) - \log( 1 + (K - 1)\exp(-C))\\
    & \geq - \sum_{i = 1}^{K} \alpha_i \log(\alpha_i) - (K - 1)\exp(-C)\\
    & \geq \sum_{i = 1}^{K} \alpha_i \min(C, - \log(\alpha_i)) - (K - 1)\exp(-C)
\end{align*}
The first inequality follows from the log-sum-inequality. The second inequality follows as $\sum_{i = 1}^{K} \max(\exp(-C), \beta_i)$ is maximized by setting one $\beta_i = 1$ for some $1\leq i \leq K$, while the rest are set to $0$. The second last inequality follows by $\log(1+x)\leq x$. The final inequality follows by taking a minimum with $C$ can only decrease the value.
\end{proof}

\section{Derivations of main result}
\label{sec:main}
\setlength{\abovedisplayskip}{3pt}
\setlength{\belowdisplayskip}{3pt}
\allowdisplaybreaks
\setlength{\parskip}{2pt}

As discussed in Section~\ref{sec:problem}, the objective here is to study the excess risk in Eq.~\eqref{eq:excess-ell} which has three main components, generalization error, retriever approximation error, and predictor approximation error (cf. \eqref{eq:joint-excess-bound}). 
In this section, we structure our results somewhat differently than the main body to capture the general setting of learning retriever with a fixed predictor, and vice versa. We first prove excess risk bounds for learning the retriever, then excess risk bounds for learning the predictor. Finally, we combine the results to obtain the guarantees for jointly learning the retriever and the predictor presented in the paper. For the rest of the analysis we need to specify the space of retrieved examples to define the complexity of the gap function (cf. \ref{a:gap-sobolev-main}). We recall that our retrieved samples are embedded in a compact subspace of $\mathbb{R}^{d_z}$, and  $\XCal$ is a compact subspace of $\mathbb{R}^d$. In particular, for simplicity, we assume that $\XCal \subseteq [-1, 1]^{d_x}$ and  $\ZCal \subseteq [-1, 1]^{d_z}$.

%%%%%%%%%%%%%%%%%%%%%%%%%%%%%%%%%%
%       RETRIEVER LEARNING.      %
%%%%%%%%%%%%%%%%%%%%%%%%%%%%%%%%%%
\subsection{Learning the retriever}
\label{app:learn-retriever}
We first study learning the retriever over class $\Theta$ when the predictor $\xi$ is fixed. The task of learning the retriever corresponds to minimizing the following over $\theta \in \Theta$, 
\begin{align*}
&\mathbb{E}_{(X,Y)\sim \DCal}[\mathbb{E}_{Z\sim p_{\theta}(\cdot| X)}\ell(h_{\xi}(X, Z), Y)] 
&=  \mathbb{E}_{X}\big[\mathbb{E}_{Z\sim p_{\theta}(\cdot| X)}\mathbb{E}_{Y|X}\ell(h_{\xi}(X, Z), Y)|X ]\big]
=  \mathbb{E}_{X}\big[\mathbb{E}_{Z\sim p_{\theta}(\cdot| X)} g_{\xi}(X, Z)\big],
\end{align*}
where $g_{\xi}(X, Z) = \mathbb{E}_{Y|X}\ell(h_{\xi}(X, Z), Y)$. 
We have a closed form for the optimal retriever when not restricted within a function class. The optimal retriever is 
$p^{\ast, \xi}(z|x) = \mathbbm{1}_{\argmin_{z'\in \ICal} g_{\xi}(x, z')}(z)$,
where a tie is broken arbitrarily.

For the fixed predictor $\xi$, let $\hat{\theta}(\xi)$ minimize the empirical risk given, and $\theta(\xi)$ minimize the population risk over the class $\Theta$, i.e.  
\begin{align*}
 &\hat{\theta}(\xi) = \argmin_{\theta \in \Theta} \frac{1}{n}\sum_{i \in [n]}\sum_{z \in \ICal}p_{\theta}(z|x_i)\ell\big(h_{\xi}(x_i, z), y_i\big), \\
 &\theta(\xi) = \argmin_{\theta \in \Theta} \mathbb{E}_{X}\big[\mathbb{E}_{Z\sim p_{\theta}(\cdot| X)} g_{\xi}(X, Z)\big].
\end{align*}
Here, the probability is defined using the softmax operator for a given $\theta \in \Theta$ as follows:
$$
p_{\theta, \ICal}\big(z | x\big) = \frac{\exp\big(r_{\theta}(x, z)\big)}{\sum_{z' \in \ICal}\exp\big(r_{\theta}(x, z')\big)}, \quad \forall~z\in\ICal, x \in \XCal.
$$

\paragraph{Hardness of retrieval:} We recall the Sobolev space with $\kappa$ derivatives as defined in Section~\ref{sec:prelims}.  The following is the restatement of Assumption~\ref{a:gap-sobolev-main} but  for any $\xi \in \Xi$ and not just the optimal one $\xi^*$. 
\begin{assumption}[Complexity of $\mathrm{g}_{\xi}$]\label{a:gap-sobolev}
For any $\xi \in \Xi$, there exists a baseline $b_{\xi}: [-1, 1]^{d_x} \to \mathbb{R}$ such that the function $\mathrm{gap}_{\xi}:[-1, 1]^{d_x + d_z}  \to \mathbb{R}$ with baseline $b_{\xi}$, as defined by $\mathrm{gap}_{\xi}(x, z) = (g_{\xi}(x,z) - b_{\xi}(x))$  lies in the Sobolev space with $\kappa$ derivatives and $L_{\infty}([-1, 1]^{d_x + d_z})$ norm.
\end{assumption}
As noted in the main text this means that the predictor loss has a possibly `complex' component $b_{\xi}(x)$, and a relatively `smooth' component $gap_{\xi}(x, z)$ that ensures two retrieved examples that are close leads to similar loss for the predictor $\xi$ for any $x \in \XCal$. As $gap_{\xi}(x, z)$ solely determines the optimal retrieved set, it's smoothness defines the hardness of underlying retrieval task.

\paragraph{Excess risk decomposition:} With the fixed predictor $\xi$, excess risk in \eqref{eq:excess-ell}  takes the following form
\begin{align*}
&R_{\ell, \ICal}(\xi, \hat{\theta}(\xi)) - R_{\ell, \ICal}(f_{{\rm opt}, \ICal}^{\ell})\\
&\qquad=  \underbrace{\sum_{\theta \in \{\theta(\xi), \hat{\theta}(\xi)\} }\big|\frac{1}{n}\sum_{i \in [n]}\sum_{z \in \ICal}p_{\theta}(z|x_i)\ell\big(h_{\xi}(x_i, z), y_i\big) - \mathbb{E}_{X}\big[\mathbb{E}_{Z\sim p_{\theta}(\cdot| X)} g_{\xi}(X, Z)\big]\big|}_{\text{retriever generalization error}} \\
&\qquad+ \underbrace{R_{\ell, \ICal}(\xi, \theta(\xi))  - \mathbb{E}_{X}\big[\min_{z\in \ICal} g_{\xi}(X, z)\big]}_{\text{retriever  approximation error}}
+ \underbrace{\mathbb{E}_{X}\big[\min_{z\in \ICal} g_{\xi}(X, z)\big]  - R_{\ell, \ICal}(f_{{\rm opt}, \ICal}^{\ell})}_{\text{error from predictor $\xi$}}.
\end{align*}

\subsubsection{Generalization error}
We now proceed to bound the generalization error using the Radamacher complexity. 
With probability at least $(1- \delta)$ for any $\delta > 0$,
\begin{align}
  &\Big|\mathbb{E}_{X}\big[\mathbb{E}_{Z\sim p_{\hat{\theta}(\xi)}(\cdot| X)} g_{\xi}(X, Z)\big] - \frac{1}{n}\sum_{i \in [n]}\sum_{z \in \ICal}p_{\hat{\theta}(\xi)}(z|x_i)\ell\big(h_{\xi}(x_i, z), y_i\big)\Big| \nonumber \\
  &\qquad\leq   2\mathbb{E}_{\boldsymbol{\sigma}}\Big[\max_{\theta \in \Theta}\frac{1}{n}\sum_{i \in [n]}\sigma_i\sum_{z \in \ICal}p_{\theta}(z|x_i)\ell\big(h_{\xi}(x_i, z), y_i\big)\Big] + 3 \ell_{\max}\sqrt{\tfrac{\log(2/\delta)}{n}} \nonumber\\
  &\qquad\leq 2 \times \inf_{\varepsilon \in [0, c_{\xi}/2]}\big(4\varepsilon + \tfrac{12}{\sqrt{n}} \int_{\varepsilon}^{c_{\xi}/2} \sqrt{\log(\mathcal{N}(\Theta, \nu, \|\cdot\|_{2, [n], \xi}))}d \nu\big)  + 3 \ell_{\max}\sqrt{\tfrac{\log(2/\delta)}{n}}  \label{eq:ret_generalization}
\end{align}
Using covering number bound with chaining we obtain the final inequality, where 
$$
c_{\xi} = \sup_{\theta\in \Theta} \Big(\tfrac{1}{n}\sum_{i\in [n]} \big(\sum_{z \in \ICal}p_{\theta}(z|x_i)\ell\big(h_{\xi}(x_i, z), y_i\big)\big)^2 \Big)^{1/2},
$$
and $\mathcal{N}(\Theta, \nu, \|\cdot\|_{2, [n], \xi})$ denote the covering number of the retriever function $\Theta$ with error $\nu$ in $L_2$ norm w.r.t. the set $\{(x_i, y_i): i \in [n]\}$ and $\xi$ fixed,
$$
\|\mathbf{u}\|_{2, [n], \xi} = \Big(\tfrac{1}{n} \sum_{i\in [n]} \big(\sum_{z \in \ICal}u_{i,z}\ell\big(h_{\xi}(x_i, z), y_i\big)\big)^2 \Big)^{1/2},\forall \mathbf{u} \in \mathbb{R}^{n \times |\ICal|}.
$$
The generalization error in retriever learning depends on the covering number of $\Theta$ (which we shall see is dependent on the embedding space of the retrieved examples). 

As $\theta(\xi)$ is a fixed retriever, we do not need to take any union bound over the retriever space. Therefore, we have 
$$
\Big|\mathbb{E}_{X}\big[\mathbb{E}_{Z\sim p_{\theta(\xi)}(\cdot| X)} g_{\xi}(X, Z)\big] - \frac{1}{n}\sum_{i \in [n]}\sum_{z \in \ICal}p_{\theta(\xi)}(z|x_i)\ell\big(h_{\xi}(x_i, z), y_i\big) \Big|
\leq  3 \ell_{\max}\sqrt{\tfrac{\log(2/\delta)}{n}}.
$$

\subsubsection{Approximation error}
The approximation error for learning the retriever depends on the hardness of the function $\min_{z\in \ICal} g_{\xi}(X, z)$. We recall that this term is approximated using softmax  over $r_{\theta}(X,Z)$ (cf. \eqref{eq:rdistn}). 
 
We want to approximate the term $\min_{z\in \ICal} g_{\xi}(x, z)$ for all $x \in \XCal$, by $\sum_{z\in \ICal} p_{\theta, \ICal}(z | x) g_{\xi}(x, z)$.  
We can break down the approximation into two parts. First we show that the function  $\mathrm{softmax}(- \tau \times g_{\xi}(x, z))$ approximates $\min_{z} g_{\xi}(x, z)$ for large $\tau$.  In particular, if $\tau = O(\log(|\ICal|)/\delta)$ then softmax approximates minimum with error $\delta$ (see, \cite{mcsherry2007mechanism,epasto2020optimal}). Second, we show that $p_{\theta, \ICal}\big(z | x\big)$ can approximate  
$\mathrm{softmax}(- \tau \times g_{\xi}(x, z))$  well in $L_2$ norm.

We define 
$$\tilde{p}_{\xi}(z| x) = \frac{\exp(-\tau g_{\xi}(x, z))}{\sum_{z'}\exp(-\tau g_{\xi}(x, z'))} = \frac{\exp(-\tau (g_{\xi}(x, z) - b_{\xi}(x)))}{\sum_{z'}\exp(-\tau (g_{\xi}(x, z') - b_{\xi}(x)))}.$$
Here recall that $b_{\xi}(x)$ is the baseline function in Assumption~\ref{a:gap-sobolev-main}. An example of such baseline is the loss under the optimal retrieved sample for each $x\in \XCal$, i.e. $b_{\xi}(x) = \min_{\tilde{z}} g_{\xi}(x, \tilde{z})$.

For any $\theta \in \Theta$, we have
\begin{align*}
    &R_{\ell, \ICal}(\xi, \theta(\xi))  - \mathbb{E}_{X}\big[\min_{z\in \ICal} g_{\xi}(X, z)\big] \\
    &\qquad\overset{(i)}{\leq} R_{\ell, \ICal}(\xi, \theta)  - \mathbb{E}_{X}\big[\min_{z\in \ICal} g_{\xi}(X, z)\big] \\
    &\qquad\overset{(ii)}{=} \mathbb{E}_{X}\big[\sum_{z\in \ICal} (p_{\theta, \ICal}(z | x) - \tilde{p}_{\xi}(z| x)) g_{\xi}(x, z)\big] + \mathbb{E}_{X}\big[\sum_{z\in \ICal} \tilde{p}_{\xi}(z| x)  - \min_{z\in \ICal} g_{\xi}(x, z)\big]\\
    &\qquad\overset{(iii)}{\leq} \mathbb{E}_{X}\big[\|g_{\xi}(x, \cdot)\|_\infty \|p_{\theta, \ICal}(\cdot | x) - \tilde{p}_{\xi}(\cdot| x)\|_{1}\big] + \frac{\log(|\ICal|)}{\tau^2}\\
    &\qquad\overset{(iv)}{\leq} \mathbb{E}_{X}\big[\|g_{\xi}(x, \cdot)\|_\infty \|r_{\theta}(x, \cdot) + \tau \mathrm{gap}_{\xi}(x,\cdot)\|_{\infty}\big] + \frac{\log(|\ICal|)}{\tau^2}\\
    &\qquad\overset{(v)}{\leq} \ell_{\max} \|r_{\theta} + \tau \mathrm{gap}_{\xi}\|_{\infty} + \frac{\log(|\ICal|)}{\tau^2}
\end{align*}
In the first inequality $(i)$, we replace $\theta(\xi)$ which is the optimal retriever for predictor $\xi$ with an arbitrary retriever $\theta$. The first term in the inequality $(iii)$ uses the norm bounds for inner product, while the second term follows from Theorem 3.1 in  \citep{epasto2020optimal} (which originates from \citep{mcsherry2007mechanism}). The inequality $(iv)$ uses the fact that softmax functions over $K$ classes follow  $\|softmax(x) - softmax(y)\|_{1} \leq \|x - y\|_{\infty}$ (see \cite{4170855}). In the final inequality $(v)$, we use $\ell_{\max}$ to bound the norm of $g_{\xi}$.

As the above bound hold for any $\tau > 0$, by optimizing of $\tau$ and $\theta$ we obtain,
\begin{equation}\label{eq:ret-approx-app}
  R_{\ell, \ICal}(\xi, \theta(\xi))  - \mathbb{E}_{X}\big[\min_{z\in \ICal} g_{\xi}(X, z)\big] 
  \leq  \inf_{\theta\in \Theta} \inf_{\tau > 0}\ell_{\max} \|r_{\theta} + \tau \mathrm{gap}_{\xi}\|_{\infty} + \frac{\log(|\ICal|)}{\tau^2}.  
\end{equation}

Since the right had side in the inequality $(v)$ holds for any $\theta \in \Theta$, if there exists a $\theta \in \Theta$ such that the function $r_{\theta}(x, z)$ approximates the function $-\tau \mathrm{gap}_{\xi}(x,z)$ well we end up with small approximation error.

\subsubsection{Instantiation of MLP retriever}
\label{sssec:ret-mlp}
We consider $\Theta$ to be the class of MLP defined in Equation~\eqref{eq:mlp-def}. As we know MLP with appropriate depth and width has universal approximation properties, this choice of $\Theta$ ensures the function $r_{\theta}(x, z)$  approximates the function $-\tau \mathrm{gap}_{\xi}(x,z)$ well. To bound the excess risk of learning the retriever, we need to prove generalization error, and approximation error bounds for the MLP class.

\paragraph{Generalization error for MLP retriever:}
To bound the generalization error, we need to first bound the covering number $\mathcal{N}(\Theta, \nu, \|\cdot\|_{2, [n], \xi})$, for  $\Theta = {\rm MLP}(\mathbb{R}^{d_x + d_z}, \mathbb{R}; W, L)$. Here, $\XCal \subseteq \mathbb{R}^{d_x}$ and $\ICal \subseteq \mathbb{R}^{d_z}$ i.e., the retrieved space is embedded in $\mathbb{R}^{d_z}$. We first want to bound the covering number $\mathcal{N}(\Theta, \nu, \|\cdot\|_{2, [n], \xi})$ with a covering number of ${\rm MLP}(\mathbb{R}^{d_x + d_z}, \mathbb{R}; W, L)$. 

For a fixed data set $\mathcal{S}_n := \{(x_1, y_1), \dots, (x_n, y_n)\}$; predictor $\xi$; and two retrievers $\theta, \theta' \in \Theta$
\begin{align*}
    &\Big(\tfrac{1}{n}\sum_{i\in [n]} \big(\sum_{z \in \ICal}(p_{\theta}(z|x_i) - p_{\theta'}(z|x_i))\ell\big(h_{\xi}(x_i, z), y_i\big)\big)^2 \Big)^{1/2}\\
    & \qquad\overset{(i)}{\leq} \ell_{\max} \Big(\tfrac{1}{n}\sum_{i\in [n]} \big(\sum_{z \in \ICal} | p_{\theta}(z|x_i) - p_{\theta'}(z|x_i)| \big)^2 \Big)^{1/2}\\
    &\qquad \overset{(ii)}{\leq} \ell_{\max} \Big(\tfrac{1}{n}\sum_{i\in [n]} \big(\max_{z \in \ICal} |r_{\theta}(x_i, z) - r_{\theta'}(x_i, z)| \big)^2 \Big)^{1/2} \\
    &\qquad \overset{(iii)}{\leq} \ell_{\max} \sup_{x \in \mathcal{S}_n, z \in \ICal} |r_{\theta}(x, z) - r_{\theta'}(x, z)|
\end{align*}
Above, the inequality $(i)$ follow by upper bounding $\ell\big(h_{\xi}(x_i, z), y_i\big)$ with $\ell_{\max}$. The inequality $(ii)$ uses the fact that softmax functions over $K$ classes follow  $\|softmax(x) - softmax(y)\|_{1} \leq \|x - y\|_{\infty}$.

Let us define the norm $\|\cdot\|_{\infty, n |\ICal|}$ as 
$\|u\|_{\infty, n |\ICal|} \triangleq \sup_{x_i \in \mathcal{S}_n} \sup_{z \in \ICal}  |u_{i, z}|, ~\forall \mathbf{u} \in \mathbb{R}^{n \times |\ICal|}.$ Now consider a $\|\cdot\|_{\infty, n |\ICal|}$ norm cover of $\Theta$, $\Theta_{{\rm cov}}$ with cardinality  $\mathcal{N}(\Theta, \nu / \ell_{\max}, \|\cdot\|_{\infty, n |\ICal|})$. 

Note that, by definition, for any $\theta \in \Theta$, there exists a $\theta_{{\rm cov}}(\theta) \in \Theta_{{\rm cov}}$ such that $\sup_{x \in \mathcal{S}_n, z \in \ICal} |r_{\theta}(x, z) - r_{\theta_{{\rm cov}}(\theta)}(x, z)| \leq \nu / \ell_{\max}$. This means, that $\Theta_{{\rm cov}}$ forms a $\nu$-cover in the $\|\cdot\|_{2, [n], \xi}$ norm. In other words, we have 
$\mathcal{N}(\Theta, \nu, \|\cdot\|_{2, [n], \xi}) \leq \mathcal{N}(\Theta, \nu / \ell_{\max}, \|\cdot\|_{\infty, n |\ICal|}).$

Most existing results on covering number bounds for MLP assumes norm bounds for the MLP weights and biases. However, we do not impose such norm bounds for the MLP weights and biases. Therefore, we will use pseudo-dimension of the class $\Theta$ from~\cite{bartlett2019nearly} to bound the covering number $\mathcal{N}(\Theta, \nu, \|\cdot\|_{\infty, n |\ICal|})$ using ~\cite{zhang2023mathematical}. In particular, if the pseudo-dimension of $\Theta$ is $d_{VC}$, then we have 
$\log \mathcal{N}(\Theta, \nu, \|\cdot\|_{\infty, n |\ICal|}) \leq 1 + \log(1 + d_{VC}) + d_{VC} \log(\max\{2, e n |\ICal| /d_{VC} \nu\})$ as per  in~\cite[Theorem 5.11]{zhang2023mathematical}. From \cite[Theorem 6]{bartlett2019nearly} we know that for the class ${\rm MLP}(\mathbb{R}^{d}, \mathbb{R}; W, L)$ the pseudo-dimension is $O(L N \log(M))$, where $N  = O(L W^2)$ is the number of parameters, and $M = O(L W)$ is the number of computation units.  By setting $\varepsilon = c / \sqrt{n}$  for a constant $c$, and $\delta = 1/n$ in Equation~\eqref{eq:ret_generalization}, for large enough $L$ (we will set $L$ as a function of the data size $n$) we obtain the final generalization error as  
\begin{equation}\label{eq:ret-gen}
\Big|\mathbb{E}_{X}\big[\mathbb{E}_{Z\sim p_{\hat{\theta}(\xi)}(\cdot| X)} g_{\xi}(X, Z)\big] - \frac{1}{n}\sum_{i \in [n]}\sum_{z \in \ICal}p_{\hat{\theta}(\xi)}(z|x_i)\ell\big(h_{\xi}(x_i, z), y_i\big)\Big| = O\big(\frac{\ell_{\max}LW \sqrt{\log(LW)\log(n |\ICal|)}}{\sqrt{n}}\big). 
\end{equation}

\paragraph{Approximation error for MLP retriever: }
Our excess risk bounds closely follow the work of \cite{siegel2023optimal} which generalizes \cite{yarotsky2017error}.\footnote{We note \cite{siegel2023optimal} works with $\Omega = [0,1]^d$, and as mentioned therein, the analysis can be extended to bounded domain, e.g. $[a, b]^{d}$ which includes our setting. Furthermore, one can extend the analysis to non-integer Sobolev and Besov spaces following \cite{siegel2023optimal}.}
Under Assumption~\ref{a:gap-sobolev},  by specializing \cite[Theorem 1]{siegel2023optimal} with $p = q = \infty$  we get that
$$
\inf_{f\in {\rm MLP}(\mathbb{R}^{d_x + d_z}, \mathbb{R}; W, L)} \|f - \mathrm{gap}_{\xi}\|_{L_{\infty}(\Omega)} \leq C \|\mathrm{gap}_{\xi}\|_{W^\kappa(L_{\infty}(\Omega))} L^{-2\kappa/(d_x + d_z)}
$$
for $\Omega\in [-1,1]^{d_x + d_z}$, $W = 25 (d_x + d_z) + 31$ and $C = c(\kappa, d_x + d_z) < \infty$ (independent of L). Note that $\kappa$ is the number of derivatives of the Sobolev space under consideration in Assumption~\ref{a:gap-sobolev}. 

Therefore, under Assumption~\ref{a:gap-sobolev} for $\Theta = {\rm MLP}(\mathbb{R}^{d_x + d_z}, \mathbb{R}; 25 (d_x + d_z) + 31, L)$   we show that
\begin{equation}\label{eq:ret-approx}
    R_{\ell, \ICal}(\xi, \theta(\xi))  - \mathbb{E}_{X}\big[\min_{z\in \ICal} g_{\xi}(X, z)\big] \leq C'\ell_{\max} L^{-2\kappa/ (d_x + d_z)}  + \frac{\log(|\ICal|)}{\tau^2}.
\end{equation}

This follows from the following series of inequalities:
\begin{align*}
& R_{\ell, \ICal}(\xi, \theta(\xi))  - \mathbb{E}_{X}\big[\min_{z\in \ICal} g_{\xi}(X, z)\big] \\
& \qquad \overset{(i)}{\leq }\mathbb{E}_{X}\big[\|g_{\xi}(x, \cdot)\|_{\infty}\big]\mathbb{E}_{X}\big[ \|r_{\theta}(x, \cdot) + \tau \mathrm{gap}_{\xi}(x,\cdot)\|_{\infty}\big] + \frac{\log(|\ICal|)}{\tau^2} \\
& \qquad \overset{(ii)}{= } \tau \mathbb{E}_{X}\big[\|g_{\xi}(x, \cdot)\|_{\infty}\big] \|\tilde{r}_{\theta} - \mathrm{gap}_{\xi}\|_{L_{\infty}(\Omega)}+ \frac{\log(|\ICal|)}{\tau^2} \\
&\qquad \overset{(iii)}{\leq } C \tau \mathbb{E}_{X}\big[\|g_{\xi}(x, \cdot)\|_{\infty}\big] \|\mathrm{gap}_{\xi}\|_{W^\kappa(L_{\infty}(\Omega))} L^{-2\kappa/ (d_x + d_z)}  + \frac{\log(|\ICal|)}{\tau^2}\\
&\qquad \overset{(iv)}{\leq } C'\ell_{\max} \tau L^{-2\kappa/ (d_x + d_z)}  + \frac{\log(|\ICal|)}{\tau^2}
\end{align*}
The first inequality $(i)$ follows from Equation~\eqref{eq:ret-approx-app}. The second equality $(ii)$, replaces $\tilde{r}_{\theta} = - \tau r_{\theta}$. The inequality $(iii)$ follows by optimizing $\tilde{r}_{\theta}$ over the class $\Theta$, as we see then $- \tau r_{\theta}$ also lies in $\Theta$, and applying Theorem 1 in \cite{siegel2023optimal}. The final inequality $(iv)$ combines $C' = C \|\mathrm{gap}_{\xi}\|_{W^\kappa(L_{\infty}(\Omega))}$ and bounds $\mathbb{E}_{X}\big[\|g_{\xi}(x, \cdot)\|_{\infty}\big]  \leq \ell_{\max}$.

Note that the choice of $\tau$ is not algorithmic, we can optimize for $\tau$. In particular, we choose $\tau = c L^{- 2\kappa /3(d_x + d_z)}\log^{1/3}(|\ICal|)$ to obtain the approximation error bound as $O(\ell_{\max}L^{- 4\kappa /3(d_x + d_z)}\log^{1/3}(|\ICal|))$, where we treat the remaining terms that are independent of $\tau$ and $L$ as constants.  

\paragraph{Excess risk for MLP retriever learning:} Adding the approximation error~\eqref{eq:ret-approx}, and the generalization error~\eqref{eq:ret-gen} we bound the excess risk  as 
\begin{align}
\text{ Excess Risk} &\leq \underbrace{\mathbb{E}_{X}\big[\min_{z\in \ICal} g_{\xi}(X, z)\big]  - R_{\ell, \ICal}(f_{{\rm opt}, \ICal}^{\ell})}_{\text{error from predictor $\xi$}} + \underbrace{O(\ell_{\max}L^{- \tfrac{4\kappa}{3(d_x + d_z)}}\log^{1/3}(|\ICal|))}_{\text{retriever approximation error}} \nonumber\\
& \qquad \qquad \quad +  \underbrace{O\big(\frac{\ell_{\max}LW \sqrt{\log(LW)\log(n |\ICal|)}}{\sqrt{n}}\big)}_{\text{retriever generalization error}} \label{eq:excess-risk-mlp-ret}
\end{align}
By choosing $L = n^{\tfrac{3(d_x + d_z)}{6(d_x + d_z) + 8 \kappa}}$, and using the data-store size $|\ICal| = poly(n)$ and width $W = O(d_x + d_z)$ we obtain
\begin{align}
\text{ Excess Risk} \leq \underbrace{\mathbb{E}_{X}\big[\min_{z\in \ICal} g_{\xi}(X, z)\big]  - R_{\ell, \ICal}(f_{{\rm opt}, \ICal}^{\ell})}_{\text{error from predictor $\xi$}} 
+ \underbrace{\tilde{O}(\ell_{\max} n^{- \tfrac{2 \kappa}{3(d_x + d_z) + 4 \kappa}})}_{\text{retriever combined error}}. \label{eq:excess-risk-rate-mlp-ret}
\end{align}

%%%%%%%%%%%%%%%%%%%%%%%%%%%%%%%%%%
%       PREDICTOR LEARNING.      %
%%%%%%%%%%%%%%%%%%%%%%%%%%%%%%%%%%
\subsection{Learning the predictor}
\label{app:learn-predictor}
We now quantify the excess risk of a predictor $\xi$ for a fixed retriever $\theta$.  For a fixed retriever $\theta$, the learning task of the predictor is to minimize
\begin{align*}
&\mathbb{E}_{(X,Y)\sim \DCal_{XY}}[\mathbb{E}_{Z\sim p_{\theta}(\cdot| X)}\ell(h_{\xi}(X, Z), Y)] =  \mathbb{E}_{((X, Z), Y)\sim \DCal_{XY}\times p_{\theta}(\cdot| X)}\big[\ell(h_{\xi}(X, Z), Y)|X \big]
\end{align*}
The predictor now learns from the joint distribution $\DCal_{XY}\times p_{\theta}(\cdot| X)$. We assume that the \emph{hardness} of the classification task performed by the predictor varies with the selected retriever $\theta$. 

Similar to retriever learning in Section~\ref{app:learn-retriever},  for a fixed retriever $\theta$, the predictor that minimizes the empirical risk  $\hat{\xi}(\theta)$, and the predictor that minimizes the population risk $\xi^{\ast}(\theta)$  over the class $\Xi$ are defined as  
\begin{align*}
 &\hat{\xi}(\theta) = \argmin_{\xi \in \Xi} \frac{1}{n}\sum_{i \in [n]}\sum_{z \in \ICal}p_{\theta}(z|x_i)\ell\big(h_{\xi}(x_i, z), y_i\big),\\
 &\xi^{\ast}(\theta) = \argmin_{\xi \in \Xi} \mathbb{E}_{X}\big[\mathbb{E}_{Z\sim p_{\theta}(\cdot| X)} g_{\xi}(X, Z)\big],   
\end{align*}
where $g_{\xi}(X, Z) = \mathbb{E}_{Y|X}\ell(h_{\xi}(X, Z), Y)$.  We also define the predictor over the class $\Xi$ with `optimal' retrieval (possibly outside of $\Theta$) that minimizes the population risk as $\xi^{\ast}$ as 
$
\xi^{\ast} = \argmin_{\xi \in \Xi} \mathbb{E}_{X}\big[\min_{z \in \ICal} g_{\xi}(X, z)\big].
$

\paragraph{Usefulness of data-store:}
We start with characterization of the prediction task in the presence of the data-store $\ICal$.  We consider that there exists a score function $h_{*}: \XCal \times \ZCal \to \mathbb{R}^{|\YCal|}$ and corresponding probability distribution 
\begin{equation}\label{eq:score-function}
    p_{*}^y(x, z) = \frac{\exp(h_{*}^y(x,z))}{\sum_{y'} \exp(h_{*}^{y'}(x,z))}
\end{equation}
that approximates well $p_{\DSf_{XY}}^y(x) \triangleq \mathbb{P}_{Y\sim\DSf_{XY}}(y|X=x)$ for all $x \in \XCal$ and $y \in \YCal$. Furthermore, this score function $h_{*}$  lies coordinate wise in the Sobolev space (see Definition~\ref{def:sobolev}). The Assumption~\ref{a:true-sobolev-main} captures the above. We restate the assumption here for convenience. 
\begin{assumption}[Retrieval quality]\label{a:true-sobolev}
There exists a score function $h_{*}: \XCal \times \ZCal \to \mathbb{R}^{|\YCal|}$ such that
\begin{enumerate}
    \item for each $y \in \YCal$, the function $h_{*}^y$ (the $y$-th coordinate of $h^*$) lies in the Sobolev space with $\kappa_{\ICal}$ derivatives and finite $L_{\infty}([-1, 1]^{d_x + d_z})$ norm,
    \item for any $x \in \XCal$ there exists a retrieved example $z^*(x) \in \ICal$ such that for $p_{*}^y(x, z)$ as defined in Equation~\eqref{eq:score-function}
    $$\max_{y \in \YCal}\sup_{x\in \XCal}|p_{*}^y(x, z(x)) - p_{\DSf_{XY}}^y(x)| \leq c_{\ICal} |\ICal|^{-\gamma_{\ICal}}.$$ 
\end{enumerate}
\end{assumption}
Note that the tuple $(\gamma_{\ICal}, d_z, \kappa_{\ICal})$ defines the usefulness of the data-store $\ICal$. In particular, the higher the $\gamma_{\ICal}$ the closer the approximation, and  the higher the $\kappa_{\ICal}$ and the smaller the embedding dimension $d_z$ the `easier' the score function used for this approximation.

\paragraph{Excess risk decomposition}
The excess risk decomposition for the learned predictor $\hat{\xi}(\theta)$ takes the following form.
{\allowdisplaybreaks
\begin{align}
&R_{\ell, \ICal}(\hat{\xi}(\theta), \theta) - R_{\ell, \ICal}(f_{{\rm opt}, \ICal}^{\ell}) \nonumber \\
&\qquad\overset{(i)}{\leq} \sum_{\xi = \xi^{\ast}(\theta), \hat{\xi}(\theta)}\big|\frac{1}{n}\sum_{i \in [n]}\sum_{z \in \ICal}p_{\theta}(z|x_i)\ell\big(h_{\xi}(x_i, z), y_i\big) - \mathbb{E}_{X}\big[\mathbb{E}_{Z\sim p_{\theta}(\cdot| X)} g_{\xi}(X, Z)\big]\big| \nonumber \\
&\qquad\quad + R_{\ell, \ICal}(\xi^{\ast}(\theta), \theta) - R_{\ell, \ICal}(f_{{\rm opt}, \ICal}^{\ell}) \nonumber \\
&\qquad\overset{(ii)}{\leq}  \sum_{\xi = \xi^{\ast}(\theta), \hat{\xi}(\theta)}\big|\frac{1}{n}\sum_{i \in [n]}\sum_{z \in \ICal}p_{\theta}(z|x_i)\ell\big(h_{\xi}(x_i, z), y_i\big) - \mathbb{E}_{X}\big[\mathbb{E}_{Z\sim p_{\theta}(\cdot| X)} g_{\xi}(X, Z)\big]\big| \nonumber \\
&\qquad\quad+ \underbrace{R_{\ell, \ICal}(\xi^{\ast}(\theta), \theta) - R_{\ell, \ICal}(\xi^{\ast}, \theta)}_{\leq 0} + R_{\ell, \ICal}(\xi^{\ast}, \theta)- R_{\ell, \ICal}(f_{{\rm opt}, \ICal}^{\ell}) \nonumber \\
&\qquad\overset{(iii)}{\leq}  \underbrace{\sum_{\xi = \xi^{\ast}(\theta), \hat{\xi}(\theta)}\big|\frac{1}{n}\sum_{i \in [n]}\sum_{z \in \ICal}p_{\theta}(z|x_i)\ell\big(h_{\xi}(x_i, z), y_i\big) - \mathbb{E}_{X}\big[\mathbb{E}_{Z\sim p_{\theta}(\cdot| X)} g_{\xi}(X, Z)\big]\big|}_{\text{generalization error}} \nonumber \\
&\qquad\quad + \underbrace{R_{\ell, \ICal}(\xi^{\ast}, \theta)  - \mathbb{E}_{X}\big[\min_{z\in \ICal} g_{\xi^{\ast}}(X, z)\big]}_{\text{retriever error}} + \underbrace{\mathbb{E}_{X}\big[\min_{z\in \ICal} g_{\xi^{\ast}}(X, z)\big] - \mathbb{E}_{X}\big[\min_{z\in \ICal} g_{f_{{\rm opt}, \ICal}^{\ell}}(X, z)\big]}_{\text{predictor error}} \label{eq:pred-decomp}
\end{align}
}
Note that in the inequality $(ii)$, the predictor $\xi^{\ast}(\theta)$ which is optimised for the fixed retriever $\theta$ has lower risk compared to the predictor $\xi^{\ast}$, i.e.    $R_{\ell, \ICal}(\xi^{\ast}(\theta), \theta) \leq  R_{\ell, \ICal}(\xi^{\ast}, \theta)$.

\subsubsection{Approximation error}
We specialize our analysis for the log-loss bounded by $\ell_{\max} > 0$ given as 
\begin{equation}\label{eq:bounded-cross-entropy-loss-app}
\ell(h_{\xi}(x,z), y) = \min(\ell_{\max}, - \log(p_{\xi}(y|x,z))) =  \min(\ell_{\max}, \log(\sum_{y'\in \YCal}\exp(h^{y'}_{\xi}(x,z))) - h^{y}_{\xi}(x,z)).
\end{equation}

Note that we need to bound the predictor error $(\mathbb{E}_{X}\big[\min_{z\in \ICal} g_{\xi^{\ast}}(X, z)\big] - \mathbb{E}_{X}\big[\min_{z\in \ICal} g_{f_{{\rm opt}, \ICal}^{\ell}}(X, z)\big])$ for the bounded log-loss. We want to relate this term to the $p_{*}^y(x, z)$ (cf. Equation.\eqref{eq:score-function}) for which we have good control over its complexity.  We first need a lower bound for  
$\mathbb{E}_{X}\big[\min_{z\in \ICal} g_{f_{{\rm opt}, \ICal}^{\ell}}(X, z)\big]$ as a function of $p_{*}^y(x, z)$. We proceed as follows:
\begin{align}
    &\mathbb{E}_{X}\big[\min_{z\in \ICal} g_{f_{{\rm opt}, \ICal}^{\ell}}(X, z)\big] \nonumber\\
    &\qquad\overset{(i)}{\geq} \mathbb{E}_{X}\big[\sum_{y \in \YCal}p_{\DSf_{XY}}^y(X) \min(\ell_{\max}, -\ln(p_{\DSf_{XY}}^y(X)))\big] -  (|\YCal| - 1)\exp(-\ell_{\max}) \nonumber\\
    &\qquad\overset{(ii)}{\geq}  \mathbb{E}_{X}\big[\sum_{y \in \YCal} p_{\DSf_{XY}}^y(X) \min(\ell_{\max}, -\ln( p_{*}^y(X, z^*(X)))\big] -  (|\YCal| - 1)\exp(-\ell_{\max}) \nonumber\\
    &\quad\quad- \exp(\ell_{\max})\, \mathbb{E}_{X}\big[\max_{y \in \YCal} |p_{*}^y(X, z^*(X)) - p_{\DSf_{XY}}^y(X)|\big] \nonumber\\
    &\qquad\overset{(iii)}{\geq}  \mathbb{E}_{X}\big[\sum_{y \in \YCal} p_{\DSf_{XY}}^y(X) \min(\ell_{\max}, -\ln( p_{*}^y(X, z^*(X)))\big] -  (|\YCal| - 1)\exp(-\ell_{\max}) - c_{\ICal} |\ICal|^{- \gamma_{\ICal}} \exp(\ell_{\max}) \nonumber\\
    &\qquad\overset{(iv)}{=}  \mathbb{E}_{X}\big[g_{h_{*}}(X, z^*(X))\big] -  (|\YCal| - 1)\exp(-\ell_{\max}) - c_{\ICal} |\ICal|^{- \gamma_{\ICal}} \exp(\ell_{\max})\label{eq:g-lower-bound}
\end{align}
In the first inequality, applying Proposition~\ref{prop:bounded-celoss}  to our setting with $C = \ell_{\max}$ and $K = |\YCal|$ we obtain the lower bound. The second inequality follows from mean-value theorem as below, 
\begin{align*}
&|\min(C, - \log(x)) - \min(C, - \log(y))| 
\leq \sup_x \big\lvert\frac{\partial}{\partial x} \min(C, - \log(x))\big\rvert \times |x - y|
\leq \exp(C)|x - y|
\end{align*}
% Another alternative way 
% \begin{align*}
% &|\min(C, - \log(x)) - \min(C, - \log(y))| \\
% & \quad=  |- \log(\max(x, \exp(- C))) +  \log(\max(y, \exp(- C)))|\\
% & \quad \leq \max \big(\big\lvert 1 - \tfrac{\max(x, \exp(- C))}{\max(y, \exp(- C))} \big\rvert, \big\lvert\tfrac{\max(y, \exp(- C))}{\max(x, \exp(- C))} - 1 \big\rvert\big)\\
% & \quad \leq \frac{|\max(x, \exp(- C) - \max(y, \exp(- C)|}{\exp(-C)} \\
% & \quad \leq \exp(C)|x - y|
% \end{align*}
Next inequality $(iii)$ is obtained by Assumption~\ref{a:true-sobolev} with $z^*(x)$ is ad defined therein. The final inequality substitutes  
$g_{h^{*}}(x, z^*(x)) = \mathbb{E}_{Y|X=x}[\ell(h_{*}(x, z^*(x)), y)]$ where $h_{*}(x,z)$ is the score function used in Equation~\eqref{eq:score-function}.

We now derive an upper bound for the predictor error part of our excess risk bound in Equation~\eqref{eq:pred-decomp}. Let $\xi \in \Xi$ be an arbitrary predictor
\begin{align*}
    \text{Predictor Error}
    &\triangleq \mathbb{E}_{X}\big[\min_{z\in \ICal} g_{\xi^{\ast}}(X, z)\big] - \mathbb{E}_{X}\big[\min_{z\in \ICal} g_{f_{{\rm opt}, \ICal}^{\ell}}(X, z)\big] \\
    &\overset{(i)}{\leq} \mathbb{E}_{X}\big[\min_{z\in \ICal} g_{\xi^{\ast}}(X, z)\big] - \mathbb{E}_{X}\big[g_{h_{*}}(X, z^*(X))\big] +  (|\YCal| - 1)\exp(-\ell_{\max}) + c_{\ICal} |\ICal|^{- \gamma_{\ICal}}\exp(\ell_{\max})\\
    &\overset{(ii)}{\leq} \mathbb{E}_{X}\big[\min_{z\in \ICal} g_{\xi}(X, z)\big] - \mathbb{E}_{X}\big[g_{h_{*}}(X, z^*(X))\big] +  (|\YCal| - 1)\exp(-\ell_{\max}) + c_{\ICal} |\ICal|^{- \gamma_{\ICal}}\exp(\ell_{\max})\\
    &\overset{(iii)}{\leq} \mathbb{E}_{X}\big[g_{\xi}(X, z^*(X))\big] - \mathbb{E}_{X}\big[g_{h_{*}}(X, z^*(X))\big] +  (|\YCal| - 1)\exp(-\ell_{\max}) + c_{\ICal} |\ICal|^{- \gamma_{\ICal}}\exp(\ell_{\max})
\end{align*}
The second inequality follows by substituting the lower bound of $\mathbb{E}_{X}\big[\min_{z\in \ICal} g_{f_{{\rm opt}, \ICal}^{\ell}}(X, z)\big]$ from Equation~\eqref{eq:g-lower-bound}. As $\xi^{\ast}$ optimizes $\ell$-risk over $\Xi$, we can substitute with the arbitrary predictor $\xi$ to obtain an upper bound. The final inequality is obtained by substituting $z^*(X)$ instead of minimizing with respect to $z \in \ICal$.  
Note that the final inequality holds for all $\xi \in \Xi$ as the initial choice of $\xi$ was arbitrary.

Bounding the term 
$\mathbb{E}_{X}\big[g_{\xi}(X, z^*(X))\big] - \mathbb{E}_{X}\big[g_{h_{*}}(X, z^*(X))\big]$,
is similar to bounding the $\ell$-risk for classification with the data distribution $\mathbb{P}(X = x, Z = z , Y = y) = \mathbb{P}_{\DSf_{XY}}(X = x, Y = y)\mathbbm{1}(z = z^*(X))$. Our strategy is to bound $\ell$-risk with $L_{\infty}$ distance between the score functions $h_{\xi}^y(x, z)$ and the score function $h_{*}^y(x, z)$ which lies in the Sobolev space as given in the Assumption~\ref{a:true-sobolev}. In particular, we have the following $L_{\infty}$ norm bound. 
\begin{align*}
    &\mathbb{E}_{X}\big[g_{\xi}(X, z^*(X))\big] - \mathbb{E}_{X}\big[g_{h_{*}}(X, z^*(X))\big]\\
    &\qquad \overset{(i)}{=} \mathbb{E}_{XY}\big[ \ell(h_{\xi}^Y(X, z^*(X))) - \ell(h_{*}^Y(X, z^*(X)))\big] \\
    &\qquad \overset{(ii)}{\leq} \mathbb{E}_{XY}\big[ |h_{\xi}^Y(X, z^*(X)) - h_{*}^Y(X, z^*(X))|  + 
    \max_{y \in \YCal} |h_{\xi}^y(X, z^*(X)) - h_{*}^y(X, z^*(X))|\big]\\
    &\qquad \overset{(iii)}{\leq} 2\times  \mathbb{E}_{X}\big[ \max_{y \in \YCal} |h_{\xi}^y(X, z^*(X)) - h_{*}^y(X, z^*(X))|\big]
\end{align*}
The inequality $(ii)$ follows by substituting the bounded log-loss, and using the fact that for any two $s, s' \in \mathbb{R}^K$, $|\log(\sum_{k} \exp(s_k)) - \log(\sum_{k} \exp(s'_k))| \leq \max_{k} |s_k - s'_k|$.
The final inequality $(iii)$ follows by bounding the first term by second.

We note that the above holds for all $\xi$. This gives the general approximation error bound as
\begin{equation}
\label{eq:pred-approx-app}
\text{Predictor Error} \leq \inf_{\xi \in \Xi}2  \mathbb{E}_{X}\big[ \max_{y \in \YCal} |h_{\xi}^y(X, z^*(X)) - h_{*}^y(X, z^*(X))|\big] +  (|\YCal| - 1)\exp(-\ell_{\max}) + c_{\ICal} |\ICal|^{- \gamma_{\ICal}} \exp(\ell_{\max}).
\end{equation}

Note the predictor approximation error is independent of retriever learning as it is compared with respect to the Bayes optimal retriever (i.e. $\min_{z \in \ICal}g_{\xi}(x,z)$) as seen in the Equation~\eqref{eq:pred-decomp}.

\subsubsection{Generalization error} The generalization error in Equation~\eqref{eq:pred-decomp} can be bounded in a similar manner as the retriever learning in Section~\ref{app:learn-retriever}. Note that the predictor is learnt over the space $\Xi$ while the retriever is fixed in this setup.
\begin{align*}
    &|\mathbb{E}_{X}\big[\mathbb{E}_{Z\sim p_{\theta}(\cdot| X)} g_{\hat{\xi}(\theta)}(X, Z)\big] -  \frac{1}{n}\sum_{i \in [n]}\sum_{z \in \ICal}p_{\theta}(z|x_i)\ell\big(h_{\hat{\xi}(\theta)}(x_i, z), y_i\big)|\\
    &\qquad\overset{(i)}{\leq}  2\mathbb{E}_{\boldsymbol{\sigma}}\Big[\max_{\xi \in \Xi}\frac{1}{n}\sum_{i \in [n]}\sigma_i\sum_{z \in \ICal}p_{\theta}(z|x_i)\ell\big(h_{\xi}(x_i, z), y_i\big)\Big] + 3 \ell_{\max}\sqrt{\tfrac{\log(2/\delta)}{n}}\\
    &\qquad\overset{(ii)}{\leq}  2 \times \inf_{\varepsilon \in [0, c_{\theta}/2]}\big(4\varepsilon + \tfrac{12}{\sqrt{n}} \int_{\varepsilon}^{c_{\theta}/2} \sqrt{\log(\mathcal{N}(\Xi, \nu, \|\cdot\|_{2, [n], \theta}))}d \nu\big)  + 3 \ell_{\max}\sqrt{\tfrac{\log(2/\delta)}{n}}
\end{align*}
The final inequality again follows using covering number based bounds with chaining (cf. ~\cite{shalev2014understanding}). 
We have used for a fixed retriever $\theta$
$$
c_{\theta} = \sup_{\xi \in \Xi} \Big(\tfrac{1}{n}\sum_{i\in [n]} \big(\sum_{z \in \ICal}p_{\theta}(z|x_i)\ell\big(h_{\xi}(x_i, z), y_i\big)\big)^2 \Big)^{1/2},
$$
and $\mathcal{N}(\Xi, \nu, \|\cdot\|_{2, [n], \theta})$ denote the covering number of the predictor function class $\Xi$ with error $\nu$ in $L_2$ norm w.r.t. the set $\{(x_i, y_i): i \in [n]\}$ and fixed $\theta$,
$$
\|\mathbf{u}\|_{2, [n], \theta} := \Big(\tfrac{1}{n}\sum_{i\in [n]} \big(\sum_{z \in \ICal}p_{\theta}(z|x_i)u_{i,z}\big)^2 \Big)^{1/2}, \, \forall \mathbf{u} \in \mathbb{R}^{n \times |\ICal|}.
$$

As $\xi^{\ast}(\theta)$ is fixed for a fixed $\theta$, we can directly bound without any union over the learner/predictor space,
$$
|\mathbb{E}_{X}\big[\mathbb{E}_{Z\sim p_{\theta}(\cdot| X)} g_{\xi^{\ast}(\theta)}(X, Z)\big] -  \frac{1}{n}\sum_{i \in [n]}\sum_{z \in \ICal}p_{\theta}(z|x_i)\ell\big(h_{\xi^{\ast}(\theta)}(x_i, z), y_i\big)| \leq 3 \ell_{\max}\sqrt{\tfrac{\log(2/\delta)}{n}}.
$$

\subsubsection{Instantiation of MLP predictor}
\label{sssec:pred-mlp}
As a concrete example, we now consider the space $\Xi = {\rm MLP}(\mathbb{R}^{d_x + d_z}, \mathbb{R}^{\YCal}; W, L)$ as defined in Equation~\eqref{eq:mlp-def}. 

\paragraph{Approximation error of MLP predictor:} Our approximation results rely mainly on the results in \cite{siegel2023optimal}. The key difference here is the output is now $|\YCal|$ dimensional. We find an MLP of depth $L$ and width at most $W' = O(d_x + d_z)$ to individually approximate the functions $h_{*}^y(x, z)$ for each $y \in \YCal$. Later we can join these networks in parallel to obtain a final network with depth $L$ and width at most $O((d_x + d_z)|\YCal|)$. In principle, these networks may share sub-networks (e.g. the bit extraction networks, the sub-domain indexation network for $p = q$ in \cite{siegel2023optimal}) used for constructing the approximation. However, this is out of scope for this work, and we leave this open.

From Theorem 1 in \cite{siegel2023optimal}, by taking $p = q = \infty$ in the theorem statement, under Assumption~\ref{a:true-sobolev} we get that for each $y \in \YCal$ there exists a MLP $f_y \in {\rm MLP}(\mathbb{R}^{d_x + d_z}, \mathbb{R}; W, L)$ such that 
$$
\|f_y - h_{*}^y\|_{L_{\infty}(\Omega)} \leq C_y \|h_{*}^y\|_{W^\kappa(L_{\infty}(\Omega))} L^{-2\kappa_{\ICal}/(d_x + d_z)}
$$
for $\Omega\in [-1,1]^{d_x + d_z}$, $W = 25 (d_x + d_z) + 31$ and $C_y = c(\kappa_{\ICal}, d_x + d_z) < \infty$ (independent of L).
By concatenating the networks $f_y$ for $y \in \YCal$ in parallel (c.f. Lemma 5 in  \cite{siegel2023optimal}), and using the first layer to share the $(d_x + d_z)$ input to these parallel networks we obtain a MLP $f_{{\rm opt}} \in {\rm MLP}(\mathbb{R}^{d_x + d_z}, \mathbb{R}^K; W_{\YCal}, L + 1)$, $W_{\YCal} = O(|\YCal|(d_x + d_z))$,  such that we have
$$
\|f^y_{{\rm opt}} - h_{*}^y\|_{L_{\infty}(\Omega)} \leq \big( \max_{y \in \YCal} C_y \|h_{*}^y\|_{W^\kappa(L_{\infty}(\Omega))}\big) L^{-2\kappa_{\ICal}/(d_x + d_z)}.
$$ 

By using $\xi = f^y_{{\rm opt}}$  in our bounds we obtain the predictor error as 
\begin{align}
\label{eq:pred-mlp-approx-app}
    \text{Predictor Error} \leq 2 \big( \max_{y \in \YCal} C_y \|h_{*}^y\|_{W^\kappa(L_{\infty}(\Omega))}\big) L^{-2\kappa_{\ICal}/(d_x + d_z)} +  (|\YCal| - 1)\exp(-\ell_{\max}) + c_{\ICal} |\ICal|^{- \gamma_{\ICal}}
\end{align}

\paragraph{Generalization error for MLP predictor:}
We now bound the generalization error in Equation~\ref{eq:pred-decomp}  when $\Xi$ denotes a class of multi-layer perceptron (MLP) with Relu nonlinearity ${\rm MLP}(\mathbb{R}^{(d_x + d_z)}, \mathbb{R}^{|\YCal|}; W, L)$. 

The first step is to bound the covering number $\mathcal{N}(\Xi, \nu, \|\cdot\|_{2, [n], \theta})$ norm with the covering number $\mathcal{N}(\Xi, \nu, \|\cdot\|_{\infty, n |\ICal| |\YCal|})$. Where  $\|\cdot\|_{\infty, n |\ICal| |\YCal|}$ is defined as $\|u\|_{\infty, n |\ICal| |\YCal|} = \sup_{x_i \in \mathcal{S}_n} \sup_{z \in \ICal} \sup_{y \in \YCal}  |u_{i, z, y}|, ~\forall \mathbf{u} \in \mathbb{R}^{n \times |\ICal| \times |\YCal|}.$

For a fixed data set $\mathcal{S}_n := \{(x_1, y_1), \dots, (x_n, y_n)\}$  and retriever $\xi$, and two predictors $\xi, \xi' \in \Xi$, we have
\begin{align*}
    &\Big(\tfrac{1}{n}\sum_{i\in [n]} \big(\sum_{z \in \ICal}p_{\theta}(z|x_i) (\ell\big(h_{\xi}(x_i, z), y_i\big) - \ell\big(h_{\xi'}(x_i, z), y_i\big))\big)^2 \Big)^{1/2}\\
    &\qquad \overset{(i)}{\leq} \Big(\tfrac{1}{n}\sum_{i\in [n]} \sum_{z \in \ICal}p_{\theta}(z|x_i) \big( \ell\big(h_{\xi}(x_i, z), y_i\big) - \ell\big(h_{\xi'}(x_i, z), y_i\big)\big)^2 \Big)^{1/2}\\
    &\qquad \overset{(ii)}{\leq} \Big(\tfrac{1}{n}\sum_{i\in [n]} \sum_{z \in \ICal}p_{\theta}(z|x_i) 
    \big(|h^{y_i}_{\xi}(x_i, z) - h^{y_i}_{\xi'}(x_i, z)| +  \max_{y \in \YCal}|h^{y}_{\xi}(x_i, z) - h^{y}_{\xi'}(x_i, z)| \big)^2 \Big)^{1/2}\\
    &\qquad \overset{(iii)}{\leq} \Big(\tfrac{1}{n}\sum_{i\in [n]} \sum_{z \in \ICal}p_{\theta}(z|x_i) 
    \big(|h^{y_i}_{\xi}(x_i, z) - h^{y_i}_{\xi'}(x_i, z)| +  \max_{y \in \YCal}|h^{y}_{\xi}(x_i, z) - h^{y}_{\xi'}(x_i, z)| \big)^2 \Big)^{1/2}\\
    &\qquad \overset{(iv)}{\leq} \sqrt{2}\sup_{x \in \XCal}\sup_{y \in \YCal}\sup_{z \in \ICal}|h^{y}_{\xi}(x, z) - h^{y}_{\xi'}(x, z)|
\end{align*}
The first inequality follows from Jensen. For the case of bounded log-loss, we obtain the second inequality using the fact that for any two $s, s' \in \mathbb{R}^K$, $|\log(\sum_{k} \exp(s_k)) - \log(\sum_{k} \exp(s'_k))| \leq \max_{k} |s_k - s'_k|$. 

Let $\Xi_{{\rm cov}}$ be a $\|\cdot\|_{\infty, n |\ICal| |\YCal| }$ norm cover  for the space $\Xi$ of cardinality $\mathcal{N}(\Xi, \nu, \|\cdot\|_{\infty, n |\ICal| |\YCal|})$. That implies, for any $\xi \in \Xi$  there exists a $\xi'(\xi) \in \Xi_{{\rm cov}}$ such that 
$\sup_{x \in \XCal}\sup_{y \in \YCal}\sup_{z \in \ICal}|h^{y}_{\xi}(x, z) - h^{y}_{\xi'}(x, z)| \leq \nu$. Therefore, due to the above inequality, we have 
$\Big(\tfrac{1}{n}\sum_{i\in [n]} \big(\sum_{z \in \ICal}p_{\theta}(z|x_i) (\ell\big(h_{\xi}(x_i, z), y_i\big) - \ell\big(h_{\xi'}(x_i, z), y_i\big))\big)^2 \Big)^{1/2}\leq \nu$. So $\Xi_{{\rm cov}}$ forms a cover of $\Xi$ with respect to the $\|\cdot\|_{2, [n], \theta}$ norm. Hence,  $\mathcal{N}(\Xi, \nu, \|\cdot\|_{2, [n], \theta}) \leq \mathcal{N}(\Xi, \nu, \|\cdot\|_{\infty, n |\ICal| |\YCal|}).$

We need to bound $\mathcal{N}(\Xi, \nu, \|\cdot\|_{\infty, n |\ICal| |\YCal|})$ next. Similar to the retrieval analysis in Section~\ref{app:learn-retriever}, we first apply ~\cite{zhang2023mathematical} to bound the covering number $\mathcal{N}(\Xi, \nu, \|\cdot\|_{\infty, n |\ICal| |\YCal|})$ with pseudo-dimension. However, we need slight reformulation of the function $h_{\xi}: \XCal \times \ZCal \to \mathbb{R}^{|\YCal|}$ to apply the results therein. Let us define function $\tilde{h}_{\xi}: \XCal \times \ZCal \times \YCal \to \mathbb{R}$, where for each $y \in \YCal$ we have $\tilde{h}_{\xi}(x, y, z) = h^y_{\xi}(x, z)$. It is easy to see that $\mathcal{N}(\Xi, \nu, \|\cdot\|_{\infty, n |\ICal| |\YCal|})$ covering of set $\Xi$ remains unchanged due to this reformulation. In particular, if the pseudo-dimension of $\{\tilde{h}_{\xi}: \xi \in \Xi\}$ is $\tilde{d}_{VC}$, then we have $\log \mathcal{N}(\Xi, \nu, \|\cdot\|_{\infty, n |\ICal| |\YCal|}) \leq 1 + \log(1 + \tilde{d}_{VC}) + \tilde{d}_{VC} \log(\max\{2, e n |\ICal| |\YCal| / \tilde{d}_{VC} \nu\})$ as per Theorem 5.11 in~\cite{zhang2023mathematical}.

Next we derive the pseudo-dimension of the class $\{\tilde{h}_{\xi}: \xi \in \Xi\}$ using ~\cite{bartlett2019nearly}. One challenge here is that for the MLP we are considering the label $y$ does not lie in the input space, rather this correspond to one coordinate of the $|\YCal|$-dimensional output. This can be captured with the slight modification of Theorem 6 in ~\cite{bartlett2019nearly}, namely Theorem~\ref{thm:multioutput-bartlett} in Appendix~\ref{sec:prelims}. By Theorem~\ref{thm:multioutput-bartlett}  we have for $\Xi = {\rm MLP}(\mathbb{R}^{d_x + d_z}, \mathbb{R}^{|\YCal|}; L, W)$ the VC dimension of $\Xi$ as ${\rm VCdim}(\Xi) = O(L\log(|\YCal|) + L^2W^2 \log(LW))$. The final generalization bound obtained  is as
\begin{equation}\label{eq:pred-gen}
  \text{Generalization Error} \leq O\bigg(\frac{\ell_{\max} \sqrt{(L\log(|\YCal|) + L^2W^2 \log(LW))\log(n |\ICal||\YCal|)}}{\sqrt{n}}\bigg).  
\end{equation}

\paragraph{Excess risk of predictor learning:} We can now combine the generalization error~\eqref{eq:pred-gen} and approximation error~\eqref{eq:pred-approx-app} to obtain the final excess risk. The final excess risk is upper bounded as 
\begin{align}
\text{ Excess Risk} &\leq \underbrace{R_{\ell, \ICal}(\xi^{\ast}, \theta)  - \mathbb{E}_{X}\big[\min_{z\in \ICal} g_{\xi^{\ast}}(X, z)\big]}_{\text{error from retriever $\theta$}} + \underbrace{O\big(L^{-2\kappa_{\ICal}/(d_x + d_z)} +  (|\YCal| - 1)\exp(-\ell_{\max}) + c_{\ICal} |\ICal|^{- \gamma_{\ICal}}\exp(\ell_{\max})\big)}_{\text{predictor approximation error}} \nonumber\\
&+  \underbrace{O\bigg(\frac{\ell_{\max} \sqrt{(L\log(|\YCal|) + L^2W^2 \log(LW))\log(n |\ICal||\YCal|)}}{\sqrt{n}}\bigg)}_{\text{predictor generalization error}} \nonumber\\
&=\underbrace{R_{\ell, \ICal}(\xi^{\ast}, \theta)  - \mathbb{E}_{X}\big[\min_{z\in \ICal} g_{\xi^{\ast}}(X, z)\big]}_{\text{error from retriever $\theta$}} 
+ \underbrace{\tilde{O}\bigg(|\YCal|^{\tfrac{2\kappa_{\ICal}}{(d_x + d_z) + 2 \kappa_{\ICal}}} n^{- \tfrac{\kappa_{\ICal}}{(d_x + d_z) + 2 \kappa_{\ICal}}}\bigg)}_{\text{predictor combined error}} \label{eq:excess-risk-rate-mlp-pred}
\end{align}
We have data store grow polynomially with data, $|\ICal| = \Omega(n^s |\YCal|^{1/\gamma_{\ICal}})$, and we let $\ell_{\max} = \log(|\YCal|) + s'\log(n)$. For $s \geq \frac{2 \kappa_{\ICal}}{((d_x + d_z) + 2 \kappa_{\ICal})\gamma_{\ICal}}$ and $s' \geq \frac{\kappa_{\ICal}}{((d_x + d_z) + 2 \kappa_{\ICal})}$, the final error bound for predictor follows by setting $L = n^{\tfrac{(d_x + d_z)}{2 (d_x + d_z) + 4 \kappa_{\ICal}}}|\YCal|^{-\frac{d_x + d_z}{(d_x + d_z) + 2 \kappa_{\ICal}}}$. Note that the choice of $L$ and $W$ here are related to predictor size, and are independent of the choices in retriever size.  Moreover, here we see Assumption~\ref{a:true-sobolev} forces the quality of retriever set to become the bottleneck in predictor excess risk, if we have  $|\ICal| = o(n^s |\YCal|^{1/\gamma_{\ICal}})$ for $s = \frac{2 \kappa_{\ICal}}{((d_x + d_z) + 2 \kappa_{\ICal})\gamma_{\ICal}}$.

%%%%%%%%%%%%%%%%%%%%%%%%
%   Joint Learning     %
%%%%%%%%%%%%%%%%%%%%%%%%

\subsection{Joint learning of retriever and predictor}
In this section, we consider the task of joint learning the predictor and retriever from the space $\Xi$ and $\Theta$, respectively.  The empirical optimizer pair $(\hat{\xi}_{\rm joint}, \hat{\theta}_{\rm joint})$ and the population optimizer $(\xi^{\ast}_{\rm joint},\theta^{\ast}_{\rm joint})$ for the joint task are given as follows.
\begin{align*}
&\hat{\xi}_{\rm joint}, \hat{\theta}_{\rm joint} = \argmin_{\xi \in \Xi, \hat{\theta}\in \Theta} \frac{1}{n}\sum_{i \in [n]}\sum_{z \in \ICal}p_{\theta}(z|x_i)\ell\big(h_{\xi}(x_i, z), y_i\big), \\
&\xi^{\ast}_{\rm joint}, \theta^{\ast}_{\rm joint}  = \argmin_{\xi \in \Xi} \mathbb{E}_{X}\big[\mathbb{E}_{Z\sim p_{\theta}(\cdot| X)} g_{\xi}(X, Z)\big].
\end{align*}
Recall, the optimal predictor with best possible retrieval is $
\xi^{\ast} = \argmin_{\xi \in \Xi} \mathbb{E}_{X}\big[\min_{z \in \ICal} g_{\xi}(X, z)\big].
$
We denote the optimal retriever for $\xi^*$ as $\theta(\xi^{\ast}) =  \argmin_{\theta \in \Theta} \mathbb{E}_{X}\big[\mathbb{E}_{Z\sim p_{\theta}(\cdot| X)} g_{\xi^{\ast}}(X, Z)\big]$.

The excess risk for the classes $\Theta$ and $\Xi$ can be bounded  as 
\begin{align*}
& R_{\ell, \ICal}(\hat{\xi}_{\rm joint}, \hat{\theta}_{\rm joint}) -  R_{\ell, \ICal}(f_{{\rm opt}, \ICal}^{\ell})\\
& \qquad\overset{(i)}{\leq} R_{\ell, \ICal}(\hat{\xi}_{\rm joint}, \hat{\theta}_{\rm joint}) - \underbrace{\bigg(R_{\ell, \ICal, n}(\hat{\xi}_{\rm joint}, \hat{\theta}_{\rm joint}) - R_{\ell, \ICal, n}(\xi^{\ast}_{\rm joint}, \theta^{\ast}_{\rm joint}) \bigg)}_{\leq 0 \text{ as ERM minimizes empirical risk} } \\
& \qquad \quad- R_{\ell, \ICal}(\xi^{\ast}_{\rm joint}, \theta^{\ast}_{\rm joint}) + R_{\ell, \ICal}(\xi^{\ast}_{\rm joint}, \theta^{\ast}_{\rm joint}) -  R_{\ell, \ICal}(f_{{\rm opt}, \ICal}^{\ell})\\
&\qquad\overset{(ii)}{\leq} \sum_{(\theta, \xi) \in \{(\hat{\theta}_{\rm joint}, \hat{\xi}_{\rm joint}), (\theta^{\ast}_{\rm joint}, \xi^{\ast}_{\rm joint})\}}|R_{\ell, \ICal}(\xi, \theta) - R_{\ell, \ICal, n}(\xi, \theta)| 
+ R_{\ell, \ICal}(\xi^{\ast}_{\rm joint}, \theta^{\ast}_{\rm joint}) -  R_{\ell, \ICal}(f_{{\rm opt}, \ICal}^{\ell})\\
&\qquad\overset{(iii)}{\leq} \sum_{(\theta, \xi) \in \{(\hat{\theta}_{\rm joint}, \hat{\xi}_{\rm joint}), (\theta^{\ast}_{\rm joint}, \xi^{\ast}_{\rm joint})\}}|R_{\ell, \ICal}(\xi, \theta) - R_{\ell, \ICal, n}(\xi, \theta)| 
+ R_{\ell, \ICal}(\xi^*, \theta(\xi^{\ast}))  -  R_{\ell, \ICal}(f_{{\rm opt}, \ICal}^{\ell})\\
&\qquad\overset{(iv)}{\leq} \underbrace{\sum_{(\theta, \xi) \in \{(\hat{\theta}_{\rm joint}, \hat{\xi}_{\rm joint}), (\theta^{\ast}_{\rm joint}, \xi^{\ast}_{\rm joint})\}}|R_{\ell, \ICal}(\xi, \theta) - R_{\ell, \ICal, n}(\xi, \theta)|}_{\text{Generalization Error}}  \\
&\qquad \quad + \underbrace{R_{\ell, \ICal}(\xi^*, \theta(\xi^{\ast})) - \mathbb{E}_{X}\big[\min_{z \in \ICal} g_{\xi^*}(X, z)\big]}_{\text{retriever error}} + \underbrace{\mathbb{E}_{X}\big[\min_{z \in \ICal} g_{\xi^*}(X, z)\big]  - R_{\ell, \ICal}(f_{{\rm opt}, \ICal}^{\ell})}_{\text{predictor error}}
\end{align*}
In the inequality $(iii)$, we substitute the pair $(\xi^*, \theta(\xi^{\ast}))$ for $(\xi^{\ast}_{\rm joint}, \theta^{\ast}_{\rm joint})$ as the former may have higher loss than latter. For the pair $(\xi^*, \theta(\xi^{\ast}))$ the predictor error is easily controlled. Also, note that the retriever $\theta(\xi^{\ast})$ is optimized for the optimal predictor $\xi^{\ast}$. Therefore, unlike the fixed predictor case in Section~\ref{app:learn-retriever} we do not have additional predictor error. We next bound the generalization and approximation errors separately by combining the retriever and predictor errors derived earlier.

\subsubsection{Generalization Error}
First, for the fixed $(\theta^{\ast}, \xi^{\ast})$ pair we bound the generalization error as 
$$
|\mathbb{E}_{X}\big[\mathbb{E}_{Z\sim p_{\theta^{\ast}}(\cdot| X)} g_{\xi^{\ast}}(X, Z)\big] -  \frac{1}{n}\sum_{i \in [n]}\sum_{z \in \ICal}p_{\theta^{\ast}}(z|x_i)\ell\big(h_{ \xi^{\ast}}(x_i, z), y_i\big)| \leq 3 \ell_{\max}\sqrt{\tfrac{\log(2/\delta)}{n}}.
$$
Next, the generalization for the $(\hat{\xi}, \hat{\theta})$ error can be bounded as.
\begin{align}
    &|\mathbb{E}_{X}\big[\mathbb{E}_{Z\sim p_{\hat{\theta}}(\cdot| X)} g_{\hat{\xi}}(X, Z)\big] -  \frac{1}{n}\sum_{i \in [n]}\sum_{z \in \ICal}p_{\hat{\theta}}(z|x_i)\ell\big(h_{\hat{\xi}}(x_i, z), y_i\big)| \nonumber\\
    &\qquad\overset{(i)}{\leq}   2\mathbb{E}_{\boldsymbol{\sigma}}\Big[\max_{(\theta, \xi) \in \Theta \times \Xi}\frac{1}{n}\sum_{i \in [n]}\sigma_i\sum_{z \in \ICal}p_{\theta}(z|x_i)\ell\big(h_{\xi}(x_i, z), y_i\big)\Big] + 3 \ell_{\max}\sqrt{\tfrac{\log(2/\delta)}{n}} \nonumber\\
    &\qquad\overset{(ii)}{\leq} 2 \times \inf_{\varepsilon \in [0, c_{\max}/2]}\big(4\varepsilon + \tfrac{12}{\sqrt{n}} \int_{\varepsilon}^{c_{\max}/2} \sqrt{\log(\mathcal{N}(\Theta \times \Xi, \nu, \|\cdot\|_{2, [n]}))}d \nu\big)  + 3 \ell_{\max}\sqrt{\tfrac{\log(2/\delta)}{n}}. \label{eq:combined-gen}
\end{align}
The second inequality again follows using covering number based bounds with chaining \cite{shalev2014understanding}. 
We have used for a fixed retriever $\theta$
$$
c_{\max} = \sup_{\theta, \xi \in \Theta \times \Xi} \Big(\sum_{i\in [n]} \big(\sum_{z \in \ICal}p_{\theta}(z|x_i)\ell\big(h_{\xi}(x_i, z), y_i\big)\big)^2 \Big)^{1/2},
$$
and $\mathcal{N}(\Xi, \nu, \|\cdot\|_{2, [n]})$ denotes the covering number of the retriever function class $\Xi$ with error $\nu$ in $L_2$ norm w.r.t. the set $\{(x_i, y_i): i \in [n]\}$, i.e.,
$$
\|\mathbf{u}\|_{2, [n]} := \Big(\sum_{i\in [n]} \big(\sum_{z \in \ICal}u_{i,z}\big)^2 \Big)^{1/2}, \, \forall \mathbf{u} \in \mathbb{R}^{n \times |\ICal|}.
$$

The covering number in Equation~\eqref{eq:combined-gen} can be bounded using the retriever and predictor learning complexities as
$$
\sqrt{\log(\mathcal{N}(\Theta \times \Xi, \nu, \|\cdot\|_{2, [n]}))}  \leq \max_{\xi \in \Xi}\sqrt{\log(\mathcal{N}(\Theta, \nu/2, \|\cdot\|_{2, [n], \xi}))} + \max_{\theta \in \Theta}\sqrt{\log(\mathcal{N}(\Xi, \nu/2, \|\cdot\|_{2, [n], \theta}))}.
$$

This implies that the generalization error of joint learning is (orderwise) bounded by the sum of the generalization error of retriever learning (cf. \eqref{eq:ret-gen}) and predictor learning (cf. \eqref{eq:pred-gen}).

\subsubsection{Approximation error}
The approximation error of predictor and retriever decouples under our decomposition, and under Assumption~\ref{a:gap-sobolev} and \ref{a:true-sobolev}. So the approximation error is also bounded by the sum of the approximation error of retriever learning with optimal predictor, and the approximation error of predictor learning.  Our derived bounds approximation error of the retriever holds uniformly for all predictor, so it also holds for optimal predictor.  
This implies that the joint retriever and predictor learning error is bounded (orderwise) by the sum of the predictor and retriever errors derived earlier in \eqref{eq:ret-approx-app}, and \eqref{eq:pred-approx-app} earlier. 

\begin{proof}[Proof of Theorem~\ref{thm:main}]
We define $f_{\mathcal{N}}(\nu; \mathcal{A}, \mathcal{B}) = \sup_{b \in \mathcal{B}}\sqrt{\log(\mathcal{N}(\mathcal{A}, \nu, \|\cdot\|_{2, n, b}))}$.
Putting the approximation and generalization errors  together we obtain the final excess risk bound as 
\begin{align*}
    &\Delta_{\ell, \xi}(\hat{\xi}, \hat{\theta}) \\
    &\quad \leq
    3 \ell_{\max} (\tfrac{1}{n} + \sqrt{\tfrac{\log(n)}{n}}) +
    \inf_{\varepsilon \in [0, \tfrac{\ell_{\max}}{2}]} 8\varepsilon + \tfrac{24}{\sqrt{n}} \int_{\varepsilon}^{\tfrac{\ell_{\max}}{2}} f_{\mathcal{N}}(\tfrac{\nu}{2}; \Theta, \Xi) + f_{\mathcal{N}}(\tfrac{\nu}{2}; \Xi, \Theta) d \nu \nonumber \\
    &\quad+  \inf_{\theta\in \Theta} \inf_{\tau > 0}\ell_{\max} \|r_{\theta} + \tau \mathrm{gap}_{\xi}\|_{\infty} + \frac{\log(|\ICal|)}{\tau^2}
    \\
    &\quad+\inf_{\xi \in \Xi}2  \mathbb{E}_{X}\big[ \max_{y \in \YCal} |h_{\xi}^y(X, z^*(X)) - h_{*}^y(X, z^*(X))|\big] +  (|\YCal| - 1)\exp(-\ell_{\max}) + c_{\ICal} |\ICal|^{- \gamma_{\ICal}}\exp(\ell_{\max}).
\end{align*}
This completes the proof.
\end{proof}

\subsubsection{Instantiation of MLP retriever and predictor} For the scenario where the retriever and predictor are MLP, we can reuse the earlier analysis to provide the excess risk bound here. 

\begin{proof}[Proof of Theorem~\ref{thm:main-mlp}]
Let us recall from Appendix~\ref{sssec:ret-mlp}, in Equation~\ref{eq:excess-risk-mlp-ret} that a retriever MLP with depth $L_{{\rm ret}}$, and width $O(d_x + d_z)$ gives an approximation error $O\left(\ell_{\max}L_{{\rm ret}}^{- \tfrac{4\kappa}{3(d_x + d_z)}}\log^{1/3}(|\ICal|)\right)$ and the generalization error $O\left(\frac{\ell_{\max}LW \sqrt{\log(LW)\log(n |\ICal|)}}{\sqrt{n}}\right)$.

Similarly, from Appendix~\ref{sssec:pred-mlp}, in  Equation~\eqref{eq:pred-mlp-approx-app}, a MLP predictor with depth $L_{{\rm pred}}$ and width $O(|\YCal|(d_x + d_z))$ has an approximation error 
$O\left(L_{{\rm pred}}^{-2\kappa_{\ICal}/(d_x + d_z)} +  (|\YCal| - 1)\exp(-\ell_{\max}) + c_{\ICal} |\ICal|^{- \gamma_{\ICal}}\exp(\ell_{\max}) \right)$, and a generalization error 
$O\left(\frac{\ell_{\max} \sqrt{(L_{{\rm pred}}\log(|\YCal|) + L_{{\rm pred}} |\YCal| \log(L_{{\rm pred}} |\YCal|))\log(n |\ICal||\YCal|)}}{\sqrt{n}}\right)$.

Thus, the combined error in this case is given as  
\begin{align*}
    \Delta_{\ell, \ICal}(\hat{\xi}, \hat{\theta}) &\leq  \tilde{O}\left(\frac{\ell_{\max}}{\sqrt{n}} \left(L_{{\rm ret}} +  L_{{\rm pred}}|\YCal| \right) \right) + 
    O\Big(\ell_{\max}L_{{\rm ret}}^{- \tfrac{4\kappa}{3(d_x + d_z)}}\log^{1/3}(|\ICal|)\Big) \\
    &+ O\left(L_{{\rm pred}}^{-\tfrac{2\kappa_{\ICal}}{(d_x + d_z)}} +  (|\YCal| - 1)\exp(-\ell_{\max}) + c_{\ICal} |\ICal|^{- \gamma_{\ICal}}\exp(\ell_{\max}) \right).
\end{align*}
This completes the proof.
\end{proof}

Finally, letting $\ell_{\max} = \log(|\YCal|) + \frac{\kappa_{\ICal}}{((d_x + d_z) + 2 \kappa_{\ICal})}\log(n)$ and combining the excess risk of retriever learning (2nd term in \eqref{eq:excess-risk-rate-mlp-ret}) and of predictor learning (2nd term in \eqref{eq:excess-risk-rate-mlp-pred}), the joint learning excess error rate is given as
\begin{align}
&\text{Joint Excess Risk MLP} \nonumber\\ 
&\qquad\leq \begin{cases}
\tilde{O}\left(n^{- \tfrac{2\kappa}{3(d_x + d_z) + 4 \kappa}} + |\YCal|^{\tfrac{2\kappa_{\ICal}}{(d_x + d_z) + 2 \kappa_{\ICal}}} n^{- \tfrac{\kappa_{\ICal}}{(d_x + d_z) + 2 \kappa_{\ICal}}}\right), 
&\text{if}~|\ICal| = \Omega\Big(|\YCal|^{\gamma_{\ICal}^{-1}} n^{\frac{2\kappa_{\ICal}\gamma_{\ICal}^{-1}}{((d_x + d_z) + 2 \kappa_{\ICal})}}\Big),\\
\tilde{O}\left(n^{- \tfrac{2\kappa}{3(d_x + d_z) + 4 \kappa}} 
+ |\ICal|^{-\gamma_{\ICal}} |\YCal| n^{\frac{\kappa_{\ICal}}{((d_x + d_z) + 2 \kappa_{\ICal})}} \right), &\text{otherwise.}
\end{cases}  \label{eq:joint-upper}  
\end{align}
Here $\kappa$ is defined in Assumption~\ref{a:gap-sobolev},
and $(\kappa_{\ICal}, \gamma_{\ICal})$ are defined in 
Assumption~\ref{a:true-sobolev}. Also, $d_x$ is the embedding dimension of input $x \in \XCal$ and $d_z$ is the 
embedding dimension of retrieved example $z \in \ICal$.

\section{More experiments}
\label{app:expt}

\begin{table*}[ht!]
    \centering
    \begin{tabular}{@{}lccc ccc ccc ccc@{}}
    \toprule
     \multicolumn{1}{@{}l}{\multirow{2.5}{*}{Method}} &  & \multicolumn{3}{c}{small} &                      & \multicolumn{3}{c}{base} &                      & \multicolumn{3}{c}{large} \\ \cmidrule(lr){3-5} \cmidrule(lr){7-9} \cmidrule(l){11-13} 
    &  & small   & base   & large  &  & small   & base  & large  &  & small   & base   & large  \\ \midrule
    % EMDR2 &  & 0.40  & 0.48  & 0.52  &  & 0.42  & 0.48  & 0.51  &  & 0.42  & 0.49  & 0.53 \\
    % PDist &  & 0.50  & 0.57  & 0.61  &  & 0.49  & 0.57  & 0.61  &  & 0.48  & 0.56  & 0.58 \\
    % Cross-Entropy + PG &  & 0.45  & 0.53  & 0.55  &  & 0.45  & 0.53  & 0.55  &  & 0.45  & 0.52  & 0.55 \\
    % Cross-Entropy + TopK &  & 0.49  & 0.57  & 0.61  &  & 0.48  & 0.55  & 0.60  &  & 0.47  & 0.54  & 0.58 \\
    EMDR2 &  & 40.0  & 47.7  & 52.0  &  & 41.5  & 48.0  & 51.4  &  & 41.6  & 48.8  & 52.6 \\
    PDist &  & 49.7  & 57.4  & 61.3  &  & 48.6  & 57.0  & 61.0  &  & 47.7  & 55.7  & 58.9 \\
    Reverse Cross-Entropy + PG &  & 44.9  & 52.6  & 54.7  &  & 45.3  & 53.3  & 55.2  &  & 44.9  & 51.7  & 54.9 \\
    Reverse Cross-Entropy + TopK &  & 48.9  & 56.8  & 60.9  &  & 47.9  & 55.5  & 59.6  &  & 46.7  & 54.3  & 58.2 \\
    \bottomrule
    \end{tabular}
    \caption{\textbf{Recall on NQ}. We measure the recall of answer string being present in the retrieved passage performance of RAMs across various training objectives and model sizes. Top row specifies the predictor size and the second row specifies the retriever size.}
    \label{tab:nq_recall}
\end{table*}

\begin{table*}[t]
    \centering
    \begin{tabular}{@{}lccc ccc ccc ccc@{}}
    \toprule
     \multicolumn{1}{@{}l}{\multirow{2.5}{*}{Method}} &  & \multicolumn{3}{c}{small} &                      & \multicolumn{3}{c}{base} &                      & \multicolumn{3}{c}{large} \\ \cmidrule(lr){3-5} \cmidrule(lr){7-9} \cmidrule(l){11-13} 
    &  & small   & base   & large  &  & small   & base  & large  &  & small   & base   & large  \\ \midrule
    EMDR2 &  & 46.6  & 54.7  & 62.4  &  & 46.1  & 55.7  & 61.6  &  & 46.0  & 53.9  & 59.5 \\
    PDist &  & 59.6  & 68.6  & 72.8  &  & 59.1  & 61.9  & 72.2  &  & 56.4  & 59.3  & 69.3 \\
    Reverse Cross-Entropy + PG &  & 58.1  & 60.7  & 70.7  &  & 56.9  & 66.1  & 64.2  &  & 54.2  & 61.4  & 61.3 \\
    Reverse Cross-Entropy + TopK &  & 57.1  & 64.5  & 69.1  &  & 55.9  & 63.5  & 68.1  &  & 54.2  & 61.2  & 65.8 \\
    \bottomrule
    \end{tabular}
    \caption{\textbf{Recall on TriviaQA}. We measure the recall of answer string being present in the retrieved passage performance of RAMs across various training objectives and model sizes. Top row specifies the predictor size and the second row specifies the retriever size.}
    \label{tab:tqa_recall}
\end{table*}

\begin{table*}[!th]
    \centering
    \begin{tabular}{@{}cccc cccc cccc@{}}
    \toprule
      \multicolumn{3}{c}{small} & &
      \multicolumn{3}{c}{base}  & &
      \multicolumn{3}{c}{large} \\ 
      \cmidrule{1-3} \cmidrule(lr){5-7} \cmidrule(l){9-11} 
    small   & base   & large  &  & small   & base  & large  &  & small   & base   & large  \\ \midrule
    96.4M  & 170.9M  & 396.4M  &  & 258.8M  & 333.3M  & 558.9M  &  & 773.6M  & 848.1M  & 1073.7M \\
    \bottomrule
    \end{tabular}
    \caption{\textbf{Parameters}. We report the model parameters in various configuration by RAMs across various  model sizes. Top row specifies the predictor size and the second row specifies the retriever size.}
    \label{tab:params}
\end{table*}

\subsection{Implementation details}
\label{sec:effcient}

Computing the objective \eqref{eq:the_obj}, let alone its gradient, requires evaluating the reader and predictor over the entire data-store $\ICal$ making it prohibitively expensive. 
We explore two ways to approximately compute the objective: 

\paragraph{Top-K approximation}
This approach involves constraining the summation to a specific subset.
Periodically we compute $p_\theta(z|x)$ for all items $z\in\ICal$ based on the current value of $\theta$.
We use this to obtain a set of $K$ documents $\ZCal(x_i)$ with the highest (stale) scores, i.e. $\mathcal{T}_K(p_\theta(\cdot|x_i))$ and evaluate the sum on this.
\begin{equation}
     \LCal^{\textsc{RCE+TopK}}_{\ICal, n}(\theta; \xi, \ICal) = -\frac{1}{n} \sum_{i\in[n]} \sum_{z \in \ZCal(x_i)} p_{\theta, \ICal}(z | x_i)\cdot \log p_{\xi}(y_i | x_i, z)
\label{eq:ce-topk}
\end{equation}
This methodology is akin to those adopted by EMDR2 and PDist, with the set being refreshed every 500 training steps and the selection of $K=64$.

\paragraph{Policy gradient}
Based on connection to RLHF/RLAIF, we propose to use policy gradient method~\citep{sutton2018reinforcement} to obtain an unbaised estimate of gradient with respect to $\theta$ efficiently.
However, as policy gradients suffer from high variance~\citep{burda2015importance,grathwohl2021oops} we use a constant baseline~\citep{williams1992simple} for variance reduction, i.e. our objective becomes
\begin{equation}
\begin{aligned}
    \LCal^{\textsc{RCE+PG}}_{\ICal, n}(\theta; \xi, \ICal) &= -\frac{1}{n} \sum_{i\in[n]} \sum_{j\in[K]} p_{\theta, \ICal}(z_j(x_i) | x_i)\cdot \big[\log p_{\xi}(y_i | x_i, z_j(x_i)) - b\big]\\
    \nabla_\theta\LCal^{\textsc{RCE+PG}}_{\ICal, n}(\theta; \xi, \ICal) &= -\frac{1}{n} \sum_{i\in[n]} \sum_{j\in[K]} \nabla_\theta \log p_{\theta, \ICal}(z_j(x_i) | x_i)\cdot \big[\log p_{\xi}(y_i | x_i, z_j(x_i)) - b\big],
\label{eq:ce-pg}
\end{aligned}
\end{equation}
where $z_j(x_i) \sim p_\theta(\cdot | x_i)$ are $K$ i.i.d. samples from the retriever distribution.
We use $K=64$ and $b=5$.

\subsection{Training details}

\textbf{Dataset} The versions of the open-domain QA datasets, we use are:
\begin{itemize}
    \item  TriviaQA: {\url{https://www.tensorflow.org/datasets/catalog/trivia_qa#trivia_qaunfilterednocontext}}
    \item NQOpen \url{https://www.tensorflow.org/datasets/catalog/natural_questions_open}
\end{itemize}

\noindent\textbf{Optimization.}~For all of our experiments, we %train all models with 
use ADAM weight decay optimizer with a short warm up period (2000 steps) and a linear decay schedule. We use the peak learning rate of $1\times10^{-4}$. The weight decay factor is 0.1. We chose batch sizes to be $64$.
The number of total training steps is as follows:
\begin{itemize}
    \item No retriever, train predictor $\xi$: 40,000
    \item Fixed retriever $\theta_0$, train predictor $\xi$: 20,000
    \item Fixed predictor $\xi^\star(\theta_0)$, train retriever $\theta$: 20,000
    \item Jointly train predictor $\xi$ and retriever $\theta$: 40,000
\end{itemize}

\noindent\textbf{Initializations}~
We initialize models for different configurations as follows:
\begin{itemize}
    \item No retriever, train predictor $\xi$: We initialize the predictor from public pretrained T5 checkpoint.
    \item Fixed retriever $\theta_0$, train predictor $\xi$: We initialize the fixed retriever from public pretrained GTR checkpoint and predictor from public pretrained T5 checkpoint.
    \item Fixed predictor $\xi^\star(\theta_0)$, train retriever $\theta$: We initialize the fixed predictor from the final checkpoint of previous run, i.e. ``Fixed retriever $\theta_0$, train predictor $\xi$''. The retriever is initialized from public pretrained GTR checkpoint.  
    \item Jointly train predictor $\xi$ and retriever $\theta$: We initialize the fixed retriever from public pretrained GTR checkpoint and predictor from public pretrained T5 checkpoint.
\end{itemize}

%%%%%%%%%%%%%%%%%%%%%%%%%%%%%%%%%%%%%%%%%%%%%%%%%%%%%%%%%%%%%%%%%%%%%%%%%%%%%%%
%%%%%%%%%%%%%%%%%%%%%%%%%%%%%%%%%%%%%%%%%%%%%%%%%%%%%%%%%%%%%%%%%%%%%%%%%%%%%%%
% OLD STUFF
%%%%%%%%%%%%%%%%%%%%%%%%%%%%%%%%%%%%%%%%%%%%%%%%%%%%%%%%%%%%%%%%%%%%%%%%%%%%%%%
%%%%%%%%%%%%%%%%%%%%%%%%%%%%%%%%%%%%%%%%%%%%%%%%%%%%%%%%%%%%%%%%%%%%%%%%%%%%%%%
% \input{old_files/050_main_result}
% \input{old_files/020_setup}
% \input{old_files/020_gen}
% \input{old_files/030_bilevel_optimization}
% \input{old_files/040_algorithm}

%%%%%%%%%%%%%%%%%%%%%%%%%%%%%%%%%%%%%%%%%%%%%%%%%%%%%%%%%%%%%%%%%%%%%%%%%%%%%%%
%%%%%%%%%%%%%%%%%%%%%%%%%%%%%%%%%%%%%%%%%%%%%%%%%%%%%%%%%%%%%%%%%%%%%%%%%%%%%%%
% AUTHOR INSTRUCTIONS
%%%%%%%%%%%%%%%%%%%%%%%%%%%%%%%%%%%%%%%%%%%%%%%%%%%%%%%%%%%%%%%%%%%%%%%%%%%%%%%
%%%%%%%%%%%%%%%%%%%%%%%%%%%%%%%%%%%%%%%%%%%%%%%%%%%%%%%%%%%%%%%%%%%%%%%%%%%%%%%

\end{document}